\PassOptionsToPackage{dvipsnames}{xcolor}
\documentclass{article} 
\usepackage[T1]{fontenc}
\usepackage[final]{colm2026_conference}

\usepackage{microtype}
\usepackage{hyperref}
\usepackage{url}
\usepackage{booktabs}
\usepackage{multirow}
\usepackage{makecell}
\usepackage{adjustbox}
\usepackage{amsmath}
\usepackage{amssymb}
\usepackage{amsthm}
\usepackage{mathtools}
\usepackage{bm}
\usepackage{dsfont}
\usepackage{standalone}
\usepackage{tikz}
\usetikzlibrary{arrows.meta,calc,shapes.arrows}
\usepackage[capitalize,noabbrev]{cleveref}
\usepackage{amsfonts}
\DeclareMathOperator*{\E}{\mathbb{E}}

\newtheorem{proposition}{Proposition}

\newcommand{\mask}{\texttt{[MASK]}}
\newcommand{\vecz}{\bm{z}}
\newcommand{\vecx}{\bm{x}}
\newcommand{\vece}{\bm{e}}
\newcommand{\vecc}{\bm{c}}
\newcommand{\vectheta}{\bm{\theta}}
\newcommand{\R}{\mathbb{R}}

\usepackage{lineno}
\usepackage[most]{tcolorbox}
\usepackage{algorithm}
\usepackage{algpseudocode}
\usepackage{subcaption}

\definecolor{creambox}{RGB}{253, 246, 236}
\definecolor{kwblue}{RGB}{0, 119, 170}
\definecolor{posteal}{RGB}{0, 121, 107}
\definecolor{tokgreen}{RGB}{46, 125, 50}
\definecolor{commentgray}{RGB}{160, 149, 133}

\definecolor{darkblue}{rgb}{0, 0, 0.5}
\hypersetup{colorlinks=true, citecolor=darkblue, linkcolor=darkblue, urlcolor=darkblue}

\title{Mask-Aware Policy Gradients for Diffusion Language Models}

\author{Haran Raajesh\thanks{Equal contribution.}, Kulin Shah\footnotemark[1], Adam Klivans, Philipp Kr\"ahenb\"uhl \\
Department of Computer Science \\
The University of Texas at Austin \\
\texttt{haranr@utexas.edu}
}

\begin{document}

\ifcolmsubmission
\linenumbers
\fi

\maketitle

\begin{abstract}
Reinforcement learning has proven effective for improving reasoning in large language models, but extending it to Masked Diffusion Language Models (MDLMs) remains challenging due to the intractability of the log-likelihood estimation.
Existing approaches approximate this log-likelihood by modeling only the token predictions, ignoring the order in which positions are unmasked during generation.
We observe that MDLM generation involves two decisions at each step: what tokens to place at each masked position and which positions to remask.
We formalize this as a two-stage action MDP, showing that the policy gradient naturally decomposes into a token term and a masking term.
Optimizing both terms achieves state-of-the-art results across mathematical reasoning and coding benchmarks, reaching \textbf{87.1\%} on GSM8K and \textbf{53.4\%} on MBPP.\footnote{Code available at \url{https://github.com/Haran71/mask-aware-policy-gradients}}
\end{abstract}

\begin{figure}[!h]
    \centering
    \includegraphics[width=1.0\linewidth]{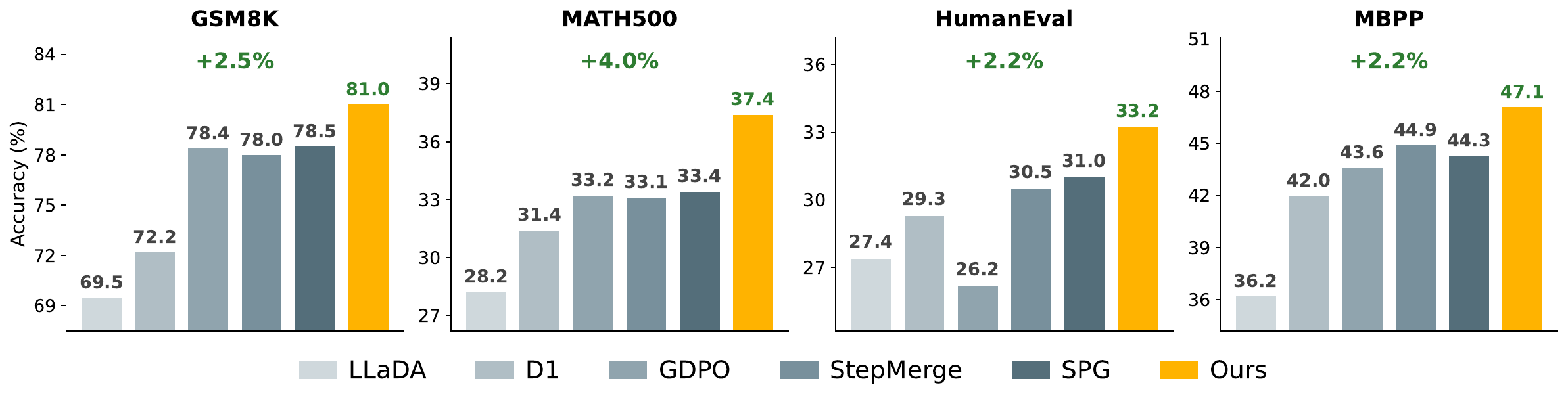}
    \caption{Test accuracy of our method and baseline methods on two mathematical reasoning and two code generation benchmarks. All methods use LLaDA-8B-Instruct as the base model and are evaluated with a generation length of 128. Full results are provided in Table~\ref{tab:main_results}.}
    \label{fig:teaser}
\end{figure}

\section{Introduction}
\label{sec:introduction}

Reinforcement learning has emerged as a powerful tool for improving large language models~\citep{deepseekr1, o1}, with applications spanning instruction following~\citep{instructgpt}, alignment~\citep{rlhf}, and complex reasoning~\citep{deepseekr1, o1}. At the core of these methods, policy gradient algorithms sample group rollouts autoregressively and estimate their log-likelihoods under the policy. The key to their success is the tractability and simplicity of the autoregressive log-likelihood: the generation trajectory is a sequence of token prediction actions, and the log-likelihood is their product.

Masked Diffusion Language Models (MDLMs)~\citep{mdlm, shi2024simplified, ou2024absorbing} offer an alternative to autoregressive LLMs, replacing the fixed left-to-right generation order with iterative unmasking from a fully masked sequence \citep{hoogeboom2022autoregressive, kim2025train}.
This enables parallel decoding and flexible generation orderings, and MDLMs have proven effective across text generation, code generation, mathematical reasoning, and protein sequence generation~\citep{lou2024discrete, mdlm, llada}.
However, this flexibility comes at a cost: unlike autoregressive models, MDLMs lack a tractable log-likelihood estimation and are instead trained via an evidence lower bound (ELBO).
This makes adapting policy gradient algorithms to MDLMs significantly more challenging.

Existing methods approximate the log-likelihood using ELBO variants~\citep{d1, gdpo, spg} or generation trajectories~\citep{d2}, focusing on the token prediction component.
However, the MDLM generation trajectory has a richer structure than its autoregressive counterpart.
While an autoregressive trajectory is a sequence of token prediction actions at fixed positions, an MDLM trajectory involves two decisions at each step: what token to place at the masked positions and which of these to remask.\footnote{We use the term \emph{remasking} following LLaDA~\citep{llada}: at each step the model predicts tokens at all masked positions, and those not selected to remain revealed are returned to $\mask$. Such positions were never actually unmasked, so this differs from self-correction methods that revert previously committed tokens.}
The trajectory log-likelihood therefore decomposes into a product of both token prediction and remasking probabilities.
Since the policy determines the unmasking order, this richer trajectory structure provides an additional signal for policy optimization that prior ELBO- and trajectory-based methods do not exploit.
See \cref{fig:intro} for an overview.

We formalize the unmasking order as part of the policy. We model MDLM generation as a two-stage action MDP: at each step, the model selects a token for each masked position and then decides which positions to remask.
This factorization yields a policy gradient that decomposes into two terms, one reinforcing token predictions and one reinforcing the ordering. The position term is computed from the model's own logits at no cost beyond the token-likelihood computation, and is compatible with any standard policy gradient algorithm. From a theoretical perspective, we show that ignoring the position component of the policy gradient can miss directions that improve expected return. Empirically, jointly optimizing both terms leads to consistent improvements over prior methods, achieving gains of up to \textbf{4.0\%} on MATH500, \textbf{2.9\%} on HumanEval, and \textbf{2.5\%} on GSM8K and MBPP (\cref{fig:teaser}).

\begin{figure}[t!]
    \centering
    \resizebox{0.9\linewidth}{!}{%
\definecolor{maskgray}{RGB}{210,210,210}%
\definecolor{newpred}{RGB}{230,230,120}%
\definecolor{kept}{RGB}{185,220,140}%
\definecolor{arrowblue1}{RGB}{170,190,235}%
\definecolor{arrowblue2}{RGB}{140,170,225}%
\definecolor{arrowblue3}{RGB}{110,150,215}%
\definecolor{arrowpurple1}{RGB}{200,175,225}%
\definecolor{arrowpurple2}{RGB}{180,150,210}%
\definecolor{arrowpurple3}{RGB}{160,130,195}%
\definecolor{verifiergray}{RGB}{170,170,170}%
\definecolor{rlboxcolor}{RGB}{245,225,190}%
\newcommand{\celltext}[4]{
    \node[cell, fill=#3] (c#1r#2) at ({#1*1.9}, {-#2*0.82}) {#4};}%
\begin{tikzpicture}[
    cell/.style={
        minimum width=1.5cm,
        minimum height=0.65cm,
        rounded corners=3pt,
        font=\small\sffamily,
        inner sep=0pt,
        outer sep=1pt,
        draw=none,
    },
    collabel/.style={font=\normalsize, anchor=north},
    arrlabel/.style={font=\footnotesize, align=center, anchor=south, inner sep=0pt},
]

\celltext{0}{0}{maskgray}{[MASK]}
\celltext{0}{1}{maskgray}{[MASK]}
\celltext{0}{2}{maskgray}{[MASK]}
\celltext{0}{3}{maskgray}{[MASK]}

\celltext{1}{0}{newpred}{A}
\celltext{1}{1}{newpred}{cat}
\celltext{1}{2}{newpred}{on}
\celltext{1}{3}{newpred}{tree}

\celltext{2}{0}{kept}{A}
\celltext{2}{1}{maskgray}{[MASK]}
\celltext{2}{2}{maskgray}{[MASK]}
\celltext{2}{3}{maskgray}{[MASK]}

\celltext{3}{0}{kept}{A}
\celltext{3}{1}{newpred}{book}
\celltext{3}{2}{newpred}{on}
\celltext{3}{3}{newpred}{diffusion}

\celltext{4}{0}{kept}{A}
\celltext{4}{1}{maskgray}{[MASK]}
\celltext{4}{2}{kept}{on}
\celltext{4}{3}{kept}{diffusion}

\celltext{5}{0}{kept}{A}
\celltext{5}{1}{newpred}{paper}
\celltext{5}{2}{kept}{on}
\celltext{5}{3}{kept}{diffusion}

\celltext{6}{0}{kept}{A}
\celltext{6}{1}{kept}{paper}
\celltext{6}{2}{kept}{on}
\celltext{6}{3}{kept}{diffusion}

\foreach \i/\lab in {0/{$\bm{z}^0$}, 1/{$\hat{\bm{z}}^1$}, 2/{$\bm{z}^1$},
                     3/{$\hat{\bm{z}}^2$}, 4/{$\bm{z}^2$}, 5/{$\hat{\bm{z}}^3$}} {
    \node[collabel] at ({1.9*\i}, {-3*0.82 - 0.55}) {\lab};
}
\node[collabel] at ({1.9*6}, {-3*0.82 - 0.55}) {$\bm{z}^3 = \bm{x}$};

\foreach \a/\b/\clr in {0/1/arrowblue1, 2/3/arrowblue2, 4/5/arrowblue3} {
    \draw[-{Stealth[length=5pt,width=4.5pt]}, color=\clr, line width=1.6pt]
        ({1.9*\a + 0.55}, 0.5) to[bend left=45] ({1.9*\b - 0.55}, 0.5);
}
\foreach \a/\b/\clr in {1/2/arrowpurple1, 3/4/arrowpurple2, 5/6/arrowpurple3} {
    \draw[-{Stealth[length=5pt,width=4.5pt]}, color=\clr, line width=1.6pt]
        ({1.9*\a + 0.55}, 0.5) to[bend left=45] ({1.9*\b - 0.55}, 0.5);
}

\foreach \a/\b/\clr in {0/1/arrowblue1, 2/3/arrowblue2, 4/5/arrowblue3} {
    \node[arrlabel, text=\clr!80!black] at ({1.9*(\a+\b)/2}, 0.92) {Sample\\Tokens};
}
\foreach \a/\b/\clr in {1/2/arrowpurple1, 3/4/arrowpurple2, 5/6/arrowpurple3} {
    \node[arrlabel, text=\clr!80!black] at ({1.9*(\a+\b)/2}, 0.92) {Remask\\Positions};
}

\coordinate (leftedge) at (-0.75, 0);
\node[anchor=west, font=\large\sffamily\itshape] at (-0.75, 1.9) {Inference};

\node[
    draw=verifiergray, fill=verifiergray!30, line width=0.8pt,
    minimum width=1.2cm,
    minimum height=3.11cm,
    rounded corners=6pt, inner sep=4pt,
] (verifier) at ({6*1.9 + 1.75}, {-1.5*0.82}) {};
\node[rotate=90, font=\small, anchor=center] at (verifier.center)
    {Verifier\quad $R(c, \bm{x})$};

\draw[-{Stealth[length=4pt]}, thick, gray!60]
    (verifier.south) |- ({6*1.9 + 0.25}, {-3*0.82 - 2.0});

\node[
    fill=rlboxcolor, rounded corners=5pt, inner sep=8pt, font=\large,
    anchor=east,
] (rlbox) at ({6*1.9 + 0.25}, {-3*0.82 - 2.0}) {%
    $R(c,\, \bm{x})\; \nabla_{\bm{\theta}} \log \pi_{\bm{\theta}}(\bm{z}^0, \hat{\bm{z}}^1, \bm{z}^1, \ldots,\, \bm{x})$%
};

\node[anchor=west, font=\large\sffamily\itshape, align=left] at (-0.75, {-3*0.82 - 2.0})
    {Reinforcement Learning\\Via Policy Gradients};

\end{tikzpicture}}
\caption{Overview of our approach. Given a prompt $\bm{c}$, a masked diffusion language model $\pi_{\bm{\theta}}$ generates text by iteratively predicting tokens and remasking positions over $T$ denoising steps, producing an extended trajectory that captures both what tokens were predicted and, through the remasking step, which positions were unmasked. A verifier scores the final output with reward $R(\bm{c}, \vecx)$, and we compute policy gradients over this full trajectory, jointly optimizing token predictions and position selection.}
    \label{fig:intro}
\end{figure}
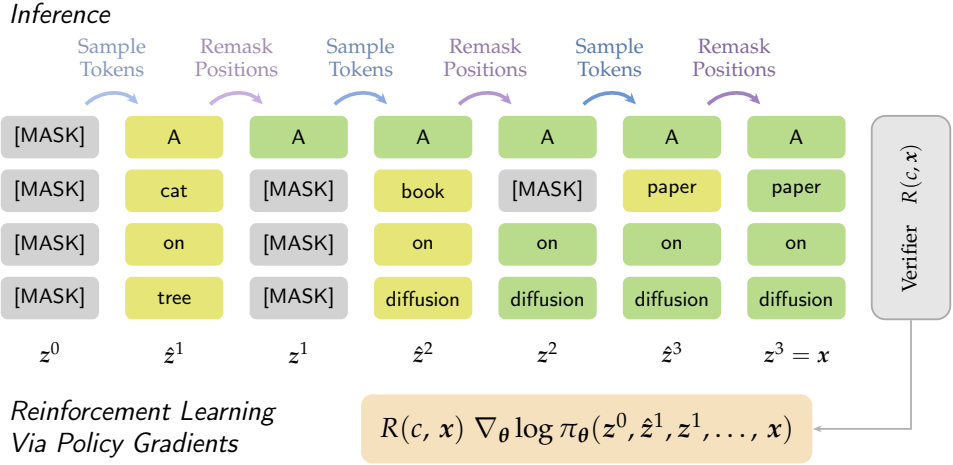

Our contributions are as follows:
\begin{itemize}
    \item We formalize MDLM generation as a two-stage action MDP and show that the policy gradient decomposes into a token term and a masking term, providing theoretical grounding for optimizing the unmasking order.
    \item We derive position log-probabilities as a softmax over unmasking scores from the model's own logits, requiring no additional parameters or architectural changes, and no forward passes beyond those already used for the token likelihood.
    \item We achieve state-of-the-art results across mathematical reasoning and code generation benchmarks.
\end{itemize}

\begin{figure}[t!]
\centering
\captionsetup{hypcap=false}
\begin{minipage}[t]{0.48\textwidth}
\begin{algorithm}[H]
\caption{Autoregressive Generation}
\label{alg:ar}
\begin{algorithmic}[1]
\Require prompt $\bm{c}$
\Statex
\Statex
\For{$k = 1$ to $n$}
    \Statex \hspace{\algorithmicindent}\textcolor{black!60}{$\triangleright$ \textit{Sample next token}}
    \State $x_k \sim \pi_{\bm{\theta}}(\cdot \mid \bm{x}_{<k}, \bm{c})$
\EndFor
\Statex
\State \Return $\bm{x}_{1:n}$
\end{algorithmic}
\end{algorithm}
\end{minipage}
\hfill
\begin{minipage}[t]{0.48\textwidth}
\begin{algorithm}[H]
\caption{MDLM Generation}
\label{alg:mdlm}
\begin{algorithmic}[1]
\Require prompt $\bm{c}$, sequence length $n$
\State $\vecz^{0} \leftarrow [\mask]^n$
\For{$t = 1 \ldots T$}
    \Statex \hspace{\algorithmicindent}\textcolor{black!60}{$\triangleright$ \textit{Predict masked positions}}
    \State $\hat{\vecz}^{t} \sim \pi_{\bm{\theta}}(\cdot \mid \vecz^{t-1}, \bm{c})$
    \Statex \hspace{\algorithmicindent}\textcolor{black!60}{$\triangleright$ \textit{Remask heuristically}}
    \State $\vecz^{t} \leftarrow \text{remask}(\hat{\vecz}^{t}, \vecz^{t-1})$
\EndFor
\State \Return $\vecz^{T}$
\end{algorithmic}
\end{algorithm}
\end{minipage}
\end{figure}

\section{Preliminaries}
\label{sec:preliminaries}

A language model $\pi_{\bm{\theta}}$ produces sequences of tokens $\bm{x}_{1:n} \sim \pi_{\bm{\theta}}(\cdot \!\mid\! \bm{c})$ conditioned on an input $\bm{c}$.
The joint distribution over sequences is exponentially large.
Autoregressive language models~\citep{deepseekr1, team2025kimi, instructgpt, o1}, decompose the distribution into a product of per-token conditionals:
\begin{equation}
    \pi_{\bm{\theta}}(\bm{x}_{1:n} \mid \bm{c}) = \prod_{k=1}^{n} \pi_{\bm{\theta}}(x_k \mid \bm{x}_{<k}, \bm{c}).
\label{eq:ar_logprob}
\end{equation}
This makes likelihood estimation tractable, and simplifies both training and generation, see \cref{alg:ar}.

\label{sec:mdlm}
\paragraph{Masked Diffusion Language Models (MDLMs)}~\citep{mdlm, shi2024simplified, ou2024absorbing} formulate text generation as an iterative denoising process.
The forward process corrupts a clean sequence $\bm{x}_{1:n}$ into $\bm{z}^t_{1:n}$ by independently replacing each token with a special $\mask$ token with probability $\alpha_t \in [0, 1]$:
\begin{equation}
    \bm{z}^t_{k} = \begin{cases} \bm{x}_k & \text{with probability } 1 - \alpha_t \\ \mask & \text{with probability } \alpha_t \end{cases}
\end{equation}
Here $t=1\ldots T$ is the diffusion step, and $\alpha_t$ is the corresponding masking probability.
We write the forward process as $\bm{z}^t_{1:n} \sim q_{t}(\cdot \mid \bm{x}_{1:n})$.
At the final step $t=T$, the sequence $\vecz^T=\bm{x}$ remains uncorrupted ($\alpha_T=0$).
At the first step $t=0$, $\vecz^0$ is fully masked ($\alpha_0=1$).

A neural network $\pi_{\bm{\theta}}$ learns to reverse this process by predicting the original tokens at masked positions from a corrupted input $\bm{z}$, maximizing the Evidence Lower Bound (ELBO) of the data log-likelihood:
\begin{equation}
    \mathcal{L}_{\mathrm{ELBO}}(\bm{x};\bm{\theta}) = \E_{\substack{t \sim \mathcal{U}\{1,\ldots,T\} \\ \bm{z}^t \sim q_{t}(\cdot \mid \bm{x})}} \bigg[\sum_{k=1}^{n} \frac{1}{\alpha_t} \mathds{1}(\bm{z}^t_{k} = \mask) \log \pi_{\bm{\theta}}(\bm{x}_k \mid \bm{z}^t) \bigg],
\label{eq:elbo}
\end{equation}
In essence, the model learns to predict the original token at each masked position, with unmasked positions left unchanged.

\paragraph{Generation in MDLMs} proceeds as described in \cref{alg:mdlm}:
We start from a completely masked sequence ${\vecz^0 = [\mask]^n}$ of length $n$ and iteratively denoise the sequence over $T$ steps with masking schedule $1 = \alpha_0 > \alpha_1 > \ldots > \alpha_T = 0$.
Each step $t$ consists of a \emph{prediction} followed by a \emph{remasking}.
First, the model predicts tokens at all masked positions:
${\hat \vecz^{t} \sim \pi_{\bm{\theta}}(\cdot \mid \vecz^{t-1})}$, with $\hat z^{t}_k = z^{t-1}_k$ for already-unmasked positions ($z^{t-1}_k \ne \mask$).
Then, a subset ${\mathcal{U}_t \subseteq \{k : z^{t-1}_k = \mask\}}$ of newly predicted positions is selected to remain unmasked, and the rest are remasked:
\begin{equation}
    z^{t}_k = \begin{cases} \hat z^{t}_k & \text{if } k \in \mathcal{U}_t \text{ or } z^{t-1}_k \ne \mask, \\ \mask & \text{otherwise.} \end{cases}
\label{eq:remask}
\end{equation}
This produces a trajectory $\mathbf{z} = (\vecz^0, \vecz^1, \ldots, \vecz^T)$ where $\vecz^T = \bm{x}_{1:n}$ is the final output.
We write $\mathbf{z} \sim \pi_{\bm{\theta}}(\cdot \!\mid\! \bm{c})$ for the distribution over trajectories induced by the model $\pi_{\bm{\theta}}$ when generating from prompt $\bm{c}$.

In theory, remasking should be random to mimic the training process.
In practice however, a deterministic strategy that selects positions to unmask based on confidence scores derived from the model's logits~\citep{chang2022maskgit, kim2025train} often performs much better.
For example, the model may select the top-$k$ positions with the highest maximum logit values to unmask at each step.
This allows the model to focus on positions it is most confident about, potentially improving generation quality and convergence speed.

\subsection{Reinforcement learning for language models}
\label{sec:rl_dllm}

RL optimizes the expected return
$$
J(\bm{\theta}) = \E_{\bm{x} \sim \pi_{\bm{\theta}}(\cdot \mid \bm{c})}[R(\bm{c}, \bm{x})].
\label{eq:pg}
$$
In the context of language models, $\bm{x}$ is the output produced by the model and $R$ is the reward obtained via a rule-based verifier or a learned reward model. Most RL algorithms for language models rely on a policy gradient estimate:
${\nabla_{\bm{\theta}} J = \mathbb{E}_{\bm{x} \sim \pi_{\bm{\theta}}} \big[ R(\bm{c}, \bm{x}) \, \nabla_{\bm{\theta}} \log \pi_{\bm{\theta}}(\bm{x} \mid \bm{c}) \big]}$.

In reinforcement learning, $\bm{x} \sim \pi_{\bm{\theta}}$ is referred to as the action.
For autoregressive models, this action is a sequence of token predictions, and $\log \pi_{\bm{\theta}}(\bm{x} \mid \bm{c})$ is directly available from \cref{eq:ar_logprob}: the chain rule decomposition gives a sum of per-token log-probabilities, computed in a single forward pass.
This makes the policy gradient fully tractable.

For MDLMs, however, $\log \pi_{\bm{\theta}}(\bm{x} \mid \bm{c})$ is \emph{intractable} to compute, as it is defined by the marginal over all trajectories that produce $\bm{x}$:
\begin{equation}
    \pi_{\bm{\theta}}(\bm{x} \mid \bm{c}) = \sum_{\mathbf{z} : \vecz^T = \bm{x}} \pi_{\bm{\theta}}(\mathbf{z} \mid \bm{c}).
\end{equation}
Existing approximations fall into two families:

\paragraph{Evidence bound methods}~\citep{d1, gdpo, spg} estimate the log-likelihood by evaluating the model on randomly masked versions of the completed sequence, typically computing some variant of the ELBO objective (Eq.~\ref{eq:elbo}) used for training.

\paragraph{Trajectory-based methods}~\citep{d2} rewrite the expected return as
\begin{equation}
    J(\bm{\theta}) =\!\!\!\E_{\bm{x} \sim \pi_{\bm{\theta}}(\cdot \mid \bm{c})}\!\!\![R(\bm{c}, \bm{x})] = \sum_{\bm{x}}\pi_{\bm{\theta}}(\bm{x} \mid \bm{c})R(\bm{c}, \bm{x}) = \sum_{\mathbf{z}} \pi_{\bm{\theta}}(\mathbf{z} \mid \bm{c})R(\bm{c}, \vecz^T) =\!\!\!\!\E_{\mathbf{z} \sim \pi_{\bm{\theta}}(\cdot \mid \bm{c})}\!\!\![R(\bm{c}, \vecz^T)].
\label{eq:exact_traj}
\end{equation}
The policy gradient estimate of the rephrased return is given by
\begin{align}
    \nabla_{\theta} & J(\bm{\theta}) \!=\! \sum_{\mathbf{z}} R(\bm{c}, \vecz^T) \frac{ \nabla_{\theta} \pi_{\bm{\theta}}(\mathbf{z} \!\mid\! \bm{c}) }{ \pi_{\bm{\theta}}(\mathbf{z} \!\mid\! \bm{c}) } \pi_{\bm{\theta}}(\mathbf{z} \!\mid\! \bm{c})    =  \E_{\mathbf{z} \sim \pi_{\bm{\theta}}(\cdot \mid \bm{c})}[ R(\bm{c}, \vecz^T) \nabla_{\theta} \log \pi_{\bm{\theta}}(\mathbf{z} \!\mid\! \bm{c}) ]
    \label{eq:pg_trajectory}
\end{align}
Unlike ELBO methods, trajectory-based methods compute an exact policy gradient estimate given we can compute $\log \pi_{\bm{\theta}}(\mathbf{z} \!\mid\! \bm{c})$ for any sampled trajectory $\mathbf{z}$:
\begin{align}
    \pi_{\theta}(\mathbf{z} \!\mid\! \vecc) = \pi_{\theta}( \vecz^1, \ldots, \vecz^T \!\mid\! \vecc) = \prod_{t=1}^{T}  \pi_{\theta}( \vecz^{t} \!\mid\! \vecc, \vecz^{\leq t-1}) = \prod_{t=1}^{T}  \pi_{\theta}( \vecz^{t} \!\mid\! \vecc, \vecz^{t-1}),
    \label{eq:traj_likelihood}
\end{align}
where the last step follows since the diffusion process is Markovian.
\citet{d2} approximate ${\log \pi_{\bm{\theta}}(\mathbf{z} \!\mid\! \bm{c}) \approx \sum_{t} \log\pi(\vecz^{t} | \vecz^{t-1}, c)}$, but ignore the remasking step.

We show that a proper treatment of this full probability estimate leads to a position term in the policy gradient estimate, which in practice greatly improves performance.


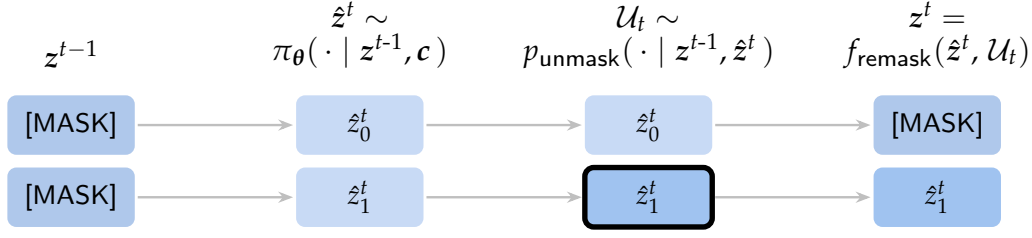
\begin{figure}[t!]
    \centering
    \resizebox{\linewidth}{!}{
\definecolor{cellblue}{RGB}{175,200,235}
\definecolor{cellbluelight}{RGB}{200,218,245}
\definecolor{highlightblue}{RGB}{160,195,240}

\sffamily
\begin{tikzpicture}[
    cell/.style={
        minimum width=1.5cm,
        minimum height=0.7cm,
        rounded corners=3pt,
        font=\small\sffamily,
        inner sep=2pt,
        outer sep=1pt,
        draw=none,
    },
    celloutline/.style={
        cell,
        draw=black,
        line width=1.8pt,
    },
    colmath/.style={font=\normalsize, anchor=south, align=center},
]

\def\colsep{3.4}
\def\colA{0}
\pgfmathsetmacro{\colB}{1*\colsep}
\pgfmathsetmacro{\colC}{2*\colsep}
\pgfmathsetmacro{\colD}{3*\colsep}
\def\rowsp{0.85}
\def\labelY{0.55}

\node[colmath] at (\colA, \labelY) {$\bm{z}^{t-1}$};
\node[cell, fill=cellblue] (a0) at (\colA, 0) {[MASK]};
\node[cell, fill=cellblue] (a1) at (\colA, {-1*\rowsp}) {[MASK]};

\node[colmath] at (\colB, \labelY) {$\hat{\bm{z}}^t \sim$\\$\pi_{\bm{\theta}}(\,\cdot \mid \bm{z}^{t\text{-}1}, \bm{c}\,)$};
\node[cell, fill=cellbluelight] (b0) at (\colB, 0) {$\hat{z}_0^t$};
\node[cell, fill=cellbluelight] (b1) at (\colB, {-1*\rowsp}) {$\hat{z}_1^t$};

\node[colmath] at (\colC, \labelY) {$\mathcal{U}_t \sim$\\$p_{\text{unmask}}(\,\cdot \mid \bm{z}^{t\text{-}1}, \hat{\bm{z}}^t\,)$};
\node[cell, fill=cellbluelight] (d0) at (\colC, 0) {$\hat{z}_0^t$};
\node[celloutline, fill=highlightblue] (d1) at (\colC, {-1*\rowsp}) {$\hat{z}_1^t$};

\node[colmath] at (\colD, \labelY) {$\bm{z}^t =$\\$f_{\text{remask}}(\hat{\bm{z}}^t,\, \mathcal{U}_t)$};
\node[cell, fill=cellblue] (e0) at (\colD, 0) {[MASK]};
\node[cell, fill=highlightblue] (e1) at (\colD, {-1*\rowsp}) {$\hat{z}_1^t$};

\draw[-{Stealth[length=5pt]}, thick, gray!50] (a0.east) -- (b0.west);
\draw[-{Stealth[length=5pt]}, thick, gray!50] (a1.east) -- (b1.west);
\draw[-{Stealth[length=5pt]}, thick, gray!50] (b0.east) -- (d0.west);
\draw[-{Stealth[length=5pt]}, thick, gray!50] (b1.east) -- (d1.west);
\draw[-{Stealth[length=5pt]}, thick, gray!50] (d0.east) -- (e0.west);
\draw[-{Stealth[length=5pt]}, thick, gray!50] (d1.east) -- (e1.west);

\end{tikzpicture}}
    \caption{Illustration of a single denoising step with probabilistic remasking. (1)~The model receives a partially masked sequence $\vecz^{t-1}$ and predicts tokens at all masked positions, producing $\hat{\vecz}^{t}$. (2)~Rather than deterministically unmasking the top-$K$ most confident positions, we sample a subset $\mathcal{U}_t$ from a Plackett--Luce distribution proportional to the predicted token log-likelihoods (shown as the bar chart). (3)~The remasking operator $f_{\text{remask}}$ retains the selected positions and remasks all others, yielding $\vecz^{t}$. In this example, position 1 is selected ($\mathcal{U}_t = \{1\}$), so only $\hat{z}_1^{t}$ is kept while position 0 is remasked.}
    \label{fig:method}
\end{figure}

\section{Method}
\label{sec:method}




We derive a policy gradient estimator for MDLMs by modeling the probability of trajectories using both token prediction and remasking in the trajectory likelihood.
Denoising from sequence $\vecz^{t-1}$ to $\vecz^{t}$ can be viewed as a two-stage process:
predicting mask tokens $\hat{\bm{z}}^t \sim \pi_\theta(\cdot\mid\vecz^{t-1})$, and selecting positions to unmask $\mathcal U_t$.
A masking function $\bm{z}^t = f_\text{remask}(\hat{\bm{z}}^t, \mathcal U_t \cup \bm{z}^{t-1} \ne \mask)$ remasks all remaining tokens.
\cref{fig:method} shows an overview.

The main issue with the above formulation is that the standard greedy top-$K$ mask $\mathcal{U}_t$ depends on $\pi_\theta$, but is not differentiable, and thus a policy gradient estimate is not possible.
We instead replace the greedy remasking with a probabilistic version both during generation, and reinforcement learning.
This leads to a proper gradient estimate, and better empirical performance.

\subsection{Probabilistic remasking}
\label{sec:position_likelihood}

Let $v^t_k = \log \pi(\hat z^t_k | \bm{z}^{t-1})$ be the log-likelihood of the sampled token $\hat z^t_k$ at position $k$ and step $t$.
In standard top-$K$ remasking, the $K$ positions with the highest scores $v^t_k$ are deterministically selected for unmasking.
This hard selection is not differentiable with respect to the model parameters $\vectheta$, preventing gradient-based optimization of the remasking.
We replace this with a \emph{probabilistic} variant: instead of selecting the top-$K$ positions deterministically, we sample $K$ positions \emph{without replacement} from a categorical distribution whose probabilities are proportional to $\exp(v^t_k / \tau)$ over all masked positions, where $\tau > 0$ is a temperature parameter.
As $\tau \to 0$, this recovers the greedy top-$K$ selection; for $\tau > 0$, it defines a proper distribution over position subsets whose log-probability can be included in the trajectory likelihood and optimized via policy gradients.
Specifically, we sample the unmasking set $\mathcal{U}_t = \{u_1, u_2, \ldots, u_K\}$ sequentially without replacement, following the Plackett--Luce model~\citep{plackett1975analysis, luce1959individual}:
${p_\text{unmask}(\mathcal{U}_t\mid \bm{z}^{t-1}, \hat{\bm{z}}^{t}) = \prod_{i=1}^K p_\text{unmask}(u_i=x | u_{<i}, \bm{z}^{t-1}, \hat{\bm{z}}^{t})}$ with
\begin{align}
p_\text{unmask}(u_i=x | u_{<i}, \bm{z}^{t-1}, \hat{\bm{z}}^{t}) = \frac{1_{[\bm{z}^{t-1}_{x}=\mask \wedge x \notin u_{<i}]} \exp(v^t_x / \tau)}{\sum_{j=1}^n 1_{[\bm{z}^{t-1}_{j}=\mask \wedge j \notin u_{<i}]} \exp(v^t_j / \tau)},
\label{eq:pos_prob}
\end{align}
where each distribution models a softmax over currently masked tokens.
Masked tokens with a higher log-likelihood $v^t_k$ have a higher chance of being unmasked.
Tokens with a low log-likelihood will likely be remasked.

The remasking probability $p_\text{remask}$ is then
\begin{align*}
    p_\text{remask}(\bm{z}^t | \mathcal{U}_t, \bm{z}^{t-1}, \hat{\bm{z}}^{t}) = 1_{[\bm{z}^t = f_\text{remask}(\hat{\bm{z}}^{t}, \mathcal{U}_t \cup \bm{z}^{t-1}=\mask)]}
\end{align*}
where $f_\text{remask}$ is the remasking operator that applies the mask $\mathcal{U}_t$ to the sequence $\bm{z}^t$:
\begin{align*}
f_\text{remask}(\hat{\bm{z}}^{t}, S)_k = \begin{cases}\hat{\bm{z}}^{t}_k&\text{if $k \in S$}\\\mask&\text{otherwise}\end{cases}
\end{align*}
Unlike greedy top-$K$ selection, this probabilistic formulation defines a proper distribution $p(\mathcal{U}_t)$ over position subsets whose log-probability is differentiable with respect to the model parameters. The remasking operator $f_\text{remask}$ and $p_\text{remask}$ are deterministic and do not directly depend on the model parameters.

\subsection{Reinforcement learning with probabilistic remasking}

The diffusion process can be viewed as producing an extended sequence ${\hat{\bm{z}} = \{\bm{z}^0, \hat{\bm{z}}^1, \mathcal{U}^1, \bm{z}^1, \ldots, \bm{z}^T\}}$ over a distribution:
\begin{equation}
    \pi_{\bm{\theta}}(\hat{\bm{z}}|\bm{c}) = \prod_{t=1}^T \pi(\hat{\bm{z}}^t | \bm{z}^{t-1}, \bm{c}) p_\text{unmask}(\mathcal{U}_t | \hat{\bm{z}}^t, \bm{z}^{t-1}, \bm{c}) p_\text{remask}(\bm{z}^t | \mathcal{U}_t, \hat{\bm{z}}^t, \bm{z}^{t-1}, \bm{c}).
\label{eq:pos_val_factorization}
\end{equation}
This interpretation extends the trajectory-based view of diffusion in \citet{d2}, and further considers the remasking process as part of the generation.
It leads to a similarly simple and direct policy gradient estimate over extended trajectories $\hat{\bm{z}}$:
\begin{equation}
\nabla_{\theta} J(\bm{\theta}) \!=\!\!\!\!\! \E_{\hat{\bm{z}} \sim \pi_{\bm{\theta}}(\cdot \mid \bm{c})}\!\! \bigg[\! R(\bm{c}, \vecz^T)\!\sum_{t=1}^{T}\!\Big( \underbrace{ \nabla_{\theta} \log p_\text{unmask}(\mathcal{U}_t | \hat{\bm{z}}^t, \bm{z}^{t-1}, \bm{c}) }_{\text{unmasking gradient}} + \underbrace{ \nabla_{\theta} \log \pi_\theta(\hat\vecz^t \!\mid\! \vecc, \vecz^{t-1} ) }_{\text{token gradient}} \Big)\!\bigg].
\label{eq:pg_decomp}
\end{equation}
Here, the remasking is deterministic and does not depend on $\bm{\theta}$ and thus falls out of the gradient estimate.
Both terms in the decomposition (Eq.~\ref{eq:pg_decomp}) are functions of the same policy $\pi_{\bm{\theta}}$.
Therefore, we can optimize them jointly with any standard policy gradient algorithm.
We use GSPO~\citep{gspo} and find it most effective.


Finally, we note that the unmasking gradient is necessary for a correct gradient estimate: optimizing only the token component, as in \citet{d2}, can miss directions that improve the objective $J(\vectheta)$.
Intuitively, changes to the unmasking order can affect the final reward even when they leave the token probabilities unchanged, so a token-only gradient provides no signal along such directions.
We make this precise in Appendix~\ref{appendix:suboptim} with a minimal counterexample: a finite-horizon MDLM problem with a linear policy in which the token-only gradient is identically zero along a direction $\bm{v}$ that strictly increases the true objective, while the full position-aware gradient is not.

\section{Experiments}
\label{sec:experiments}

\begin{table}[t!]
\centering
\adjustbox{width=\textwidth,center}{
\begin{tabular}{lcccccccccccc}
\toprule
& \multicolumn{3}{c}{\textbf{GSM8K}} & \multicolumn{3}{c}{\textbf{MATH500}} & \multicolumn{3}{c}{\textbf{HumanEval}} & \multicolumn{3}{c}{\textbf{MBPP}} \\
\cmidrule(lr){2-4} \cmidrule(lr){5-7} \cmidrule(lr){8-10} \cmidrule(lr){11-13}
\textbf{Model / Seq Len} & \textbf{128} & \textbf{256} & \textbf{512} & \textbf{128} & \textbf{256} & \textbf{512} & \textbf{128} & \textbf{256} & \textbf{512} & \textbf{128} & \textbf{256} & \textbf{512} \\
\midrule
LLaDA-8B-Inst. & 69.5 & 77.2 & 79.8 & 28.2 & 32.4 & 34.6 & 27.4 & 35.5 & 37.8 & 36.2 & 41.2 & 40.4  \\
LLaDA-1.5     & 70.4 & 80.5 & 81.9 & 26.8 & 32.2 & 35.8 & --- & --- & --- & --- & --- & --- \\
d1            & 72.2 & 80.6 & 81.3 & 31.4 & 36.0 & 39.4 & 29.3 & 39.0 & 34.8 & 42.0 & 45.5 & 41.6 \\
wd1           & 74.6 & 81.5 & 83.0 & 31.0 & 37.4 & 39.0 & --- & --- & --- & --- & --- & --- \\
UniGRPO       & 74.9 & 82.5 & 82.7 & 32.4 & 37.4 & 39.4 & --- & --- & --- & --- & --- & --- \\
GDPO          & 78.4 & 82.8 & \underline{85.0} & 33.2 & 39.6 & 41.4 & 26.2 & 39.6 & 39.0 & 43.6 & \underline{50.6} & 47.1 \\
SPG w/ Mixture & \underline{78.5} & \underline{83.9}$^\dagger$ & 84.5 & \underline{33.4} & \underline{40.0} & \underline{41.8} & \underline{31.0} & \underline{40.6} & \underline{41.1} & 44.3 & \underline{50.6} & 50.8 \\
StepMerge     & 78.0 & 83.3 & 84.8 & 33.1 & 39.1 & 41.2 & 30.5 & 39.9 & 40.7 & \underline{44.9} & 50.1 & \underline{50.9} \\
\midrule
Ours          & \makecell{\textbf{81.0}\\\textcolor{ForestGreen}{\small\textbf{+2.5}}} & \makecell{\textbf{85.9}\\\textcolor{ForestGreen}{\small\textbf{+2.0}}} & \makecell{\textbf{87.1}\\\textcolor{ForestGreen}{\small\textbf{+2.1}}} & \makecell{\textbf{37.4}\\\textcolor{ForestGreen}{\small\textbf{+4.0}}} & \makecell{\textbf{42.2}\\\textcolor{ForestGreen}{\small\textbf{+2.2}}} & \makecell{\textbf{44.2}\\\textcolor{ForestGreen}{\small\textbf{+2.4}}} & \makecell{\textbf{33.2}\\\textcolor{ForestGreen}{\small\textbf{+2.2}}} & \makecell{\textbf{43.1}\\\textcolor{ForestGreen}{\small\textbf{+2.5}}} & \makecell{\textbf{44.0}\\\textcolor{ForestGreen}{\small\textbf{+2.9}}} & \makecell{\textbf{47.1}\\\textcolor{ForestGreen}{\small\textbf{+2.2}}} & \makecell{\textbf{52.8}\\\textcolor{ForestGreen}{\small\textbf{+2.2}}} & \makecell{\textbf{53.4}\\\textcolor{ForestGreen}{\small\textbf{+2.5}}} \\
\bottomrule
\end{tabular}
}
\caption{Comparison of RL methods for diffusion language models on mathematical reasoning and code generation benchmarks. All methods use LLaDA-8B-Instruct as the base model. Results reported at generation lengths 128, 256, and 512. Best results are in \textbf{bold}, second best are \underline{underlined}. Dashes indicate results not reported in the original work. $^\dagger$Obtained after running publicly available code; the originally reported number was 86.1.}
\label{tab:main_results}
\end{table}

Following prior work~\citep{d1, gdpo, spg, wd1}, we use LLaDA-8B-Instruct~\citep{llada} as the base model for RL fine-tuning. We employ Low-Rank Adaptation (LoRA) with rank $r=128$, scaling factor $\alpha=64$, and dropout $0.05$, targeting all attention and MLP projection layers. We use 4-bit quantization (NF4) and Flash Attention 2 for memory efficiency, with all training in bfloat16.

\paragraph{Position log-prob computation.} We estimate both token and position log-probabilities using the StepMerge approximation~\citep{d2}, which groups $T$ denoising steps into $N$ contiguous segments and treats all tokens unmasked within a segment as if they were unmasked at the same step. See Appendix~\ref{appendix:stepmerge_approx} for more details on the method and the theoretical analysis of approximation error induced by using StepMerge to evaluate the trajectory
log-probability. We subsample $K$ boundaries for evaluation. The unmasking score at each masked position is the maximum logit from the model's output. Since we compute both log-probability terms from the same forward passes, the position term adds no overhead. For our main results, we set $N=32$ and $K=12$ (see Appendix~\ref{appendix:ablations}). We optimize with GSPO~\citep{gspo}, clipping the token and position importance ratios separately.

\paragraph{Datasets and evaluation.} We evaluate on GSM8K~\citep{gsm8k} and MATH500~\citep{math500} for mathematical reasoning, and HumanEval~\citep{humaneval} and MBPP~\citep{mbpp} for code generation. For math, we follow the train-test splitting, reward functions, and evaluation protocol of d1~\citep{d1}, wd1~\citep{wd1}, and SPG~\citep{spg}. For code, we follow GDPO~\citep{gdpo}, using KodCode-Light-RL-10K~\citep{kodcode} as the training dataset. All evaluations are zero-shot. For RL rollouts, we use a sequence length of 256 tokens, 128 diffusion steps, and block size 32. We report accuracy at generation lengths 128, 256, and 512.

\paragraph{Baselines.} We compare against LLaDA-8B-Instruct~\citep{llada} (base model), LLaDA-1.5~\citep{llada1.5}, d1~\citep{d1}, wd1~\citep{wd1}, UniGRPO~\citep{mmada}, GDPO~\citep{gdpo}, SPG w/ Mixture~\citep{spg}, and StepMerge~\citep{d2}. Math results for prior methods are taken from their original papers. SPG does not report code generation results, so we evaluate their official code on HumanEval and MBPP ourselves. For StepMerge, we reimplement the algorithm with the same $K$ subsampling as our method, making it a direct ablation that differs only in the absence of the masking term.

\subsection{Experimental Results}
\label{sec:results}

\paragraph{Main results.} Table~\ref{tab:main_results} presents our main results across four benchmarks. Our method achieves the best results in every configuration. The comparison against StepMerge is particularly informative: both methods share the same trajectory-based likelihood framework and $K$ subsampling, differing only in the masking log-probability term. The consistent 2-4 point improvement across all configurations demonstrates that the masking gradient provides a meaningful training signal beyond what token-level optimization alone captures. Against evidence bound methods such as GDPO and SPG, our method also shows consistent gains, particularly on MATH500 (+4.0 at length 128, +2.4 at length 512) and code generation benchmarks. We further show that our method generalizes to a second base model, Dream-7B (Appendix~\ref{appendix:dream}), and to two planning tasks, Sudoku and Countdown (Appendix~\ref{appendix:planning}).


\paragraph{Training efficiency.}
We compare the training efficiency of our method against SPG~\citep{spg}, the strongest ELBO-based baseline in our setting, on GSM8K. Both methods run on the same hardware (8 H100 GPUs) from the same base model, with identical generation, optimization, and evaluation configurations; the only difference is the log-probability estimator. 

We compare the wall-clock time to reach fixed accuracy targets both for our method and SPG in Table~\ref{tab:eff_speedup}. Our method reaches $74\%$ accuracy $1.9\times$ faster than SPG and reaches SPG's final accuracy in roughly 15 hours, compared with 18 hours for SPG. It also reaches $79\%$ and $80\%$ accuracy, which SPG does not reach. Table~\ref{tab:eff_summary} gives the efficiency summary: our per-step throughput is lower because the trajectory-based estimator requires more forward passes per optimizer step, but the position-aware gradient converges in fewer steps and reaches a higher final accuracy (81.0 vs.\ 78.5). Figure~\ref{fig:efficiency} plots accuracy against wall-clock time: our method starts marginally behind SPG, overtakes it within the first few hours, and the gap widens as training proceeds. Across three random seeds our final accuracy varies by at most $0.2$ points, so these differences are well outside run-to-run variation.

\begin{table}[t!]
\centering
\begin{subtable}[t]{0.44\textwidth}
\centering
\begin{tabular}{lcc}
\toprule
 & \textbf{SPG} & \textbf{Ours} \\
\midrule
Throughput (steps/h) & \textbf{360} & 290 \\
Convergence (steps)  & 6500 & \textbf{6000} \\
Final accuracy (\%)  & 78.5 & \textbf{81.0} \\
Match SPG-final (h)  & 18 & \textbf{15} \\
\bottomrule
\end{tabular}
\caption{Efficiency summary.}
\label{tab:eff_summary}
\end{subtable}
\hfill
\begin{subtable}[t]{0.52\textwidth}
\centering
\begin{tabular}{lccc}
\toprule
\textbf{Target} & \textbf{SPG (h)} & \textbf{Ours (h)} & \textbf{Speedup} \\
\midrule
74\% & 8.6  & 4.5  & 1.9$\times$ \\
75\% & 9.9  & 5.8  & 1.7$\times$ \\
76\% & 13.4 & 10.0 & 1.3$\times$ \\
77\% & 15.9 & 11.8 & 1.3$\times$ \\
78\% & 17.2 & 14.0 & 1.2$\times$ \\
79\% & ---  & 16.5 & --- \\
80\% & ---  & 18.3 & --- \\
\bottomrule
\end{tabular}
\caption{Wall-clock hours to reach a target accuracy.}
\label{tab:eff_speedup}
\end{subtable}
\caption{Training efficiency on GSM8K (8 H100 GPUs, generation length 128). (a) Throughput, convergence, and final accuracy, with the better value per row in \textbf{bold}. (b) Wall-clock hours to reach fixed accuracy targets, with speedup over SPG; ``---'' denotes a target SPG never reaches.}
\label{tab:efficiency}
\end{table}

\begin{figure}[t]
\centering
\includegraphics[width=0.66\linewidth]{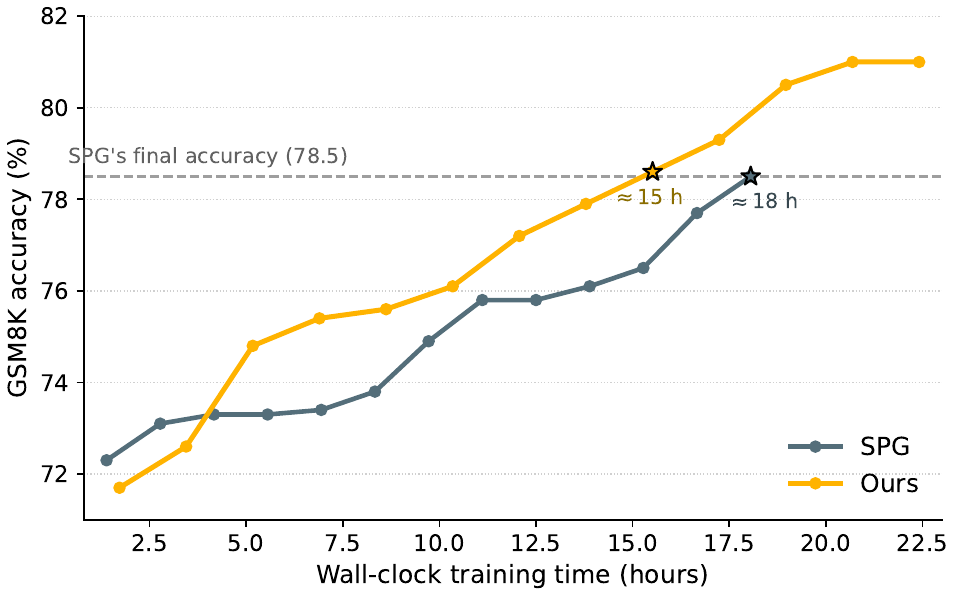}
\caption{Wall-clock training efficiency on GSM8K (8 H100 GPUs, generation length 128). Test accuracy versus wall-clock training time for our method and SPG~\citep{spg}. Although our method runs fewer steps per hour, it overtakes SPG within the first few hours and converges to a higher accuracy (81.0 vs.\ 78.5). Stars mark where each method first reaches SPG's final accuracy: $\approx$15 hours for ours versus $\approx$18 hours for SPG.}
\label{fig:efficiency}
\end{figure}

\paragraph{Inference block size.}
Figure~\ref{fig:blocksize_ablation} compares inference block sizes (32, 64, and full sequence) at generation length 256 with confidence-based unmasking.
We compare against SPG~\citep{spg}, a strong evidence-bound baseline, and StepMerge~\citep{d2}, which shares the same trajectory-based framework but without the position term.
All methods are trained with block size 32.
Our method outperforms both baselines at every block size, and the gap widens as block size increases.
Against SPG, the margin grows from +2.0 to +4.3 on GSM8K and from +2.2 to +3.5 on MATH500 (block 32 to full sequence).
Against StepMerge, the margin grows from +2.6 to +4.4 on GSM8K and from +3.1 to +3.7 on MATH500.
We attribute this to the growing importance of position selection as the diffusion window expands: larger blocks present more masked positions to choose from at each step, increasing the variability in unmasking order and amplifying the benefit of mask-aware optimization.
Full numerical results are provided in Table~\ref{tab:blocksize_ablation} in the appendix.
We provide further ablations on the RL algorithm, the StepMerge configuration, the sampling temperatures, and the unmasking-set size $|\mathcal{U}_t|$ in Appendix~\ref{appendix:ablations}.

\begin{figure}[t!]
    \centering
    \includegraphics[width=1.0\linewidth]{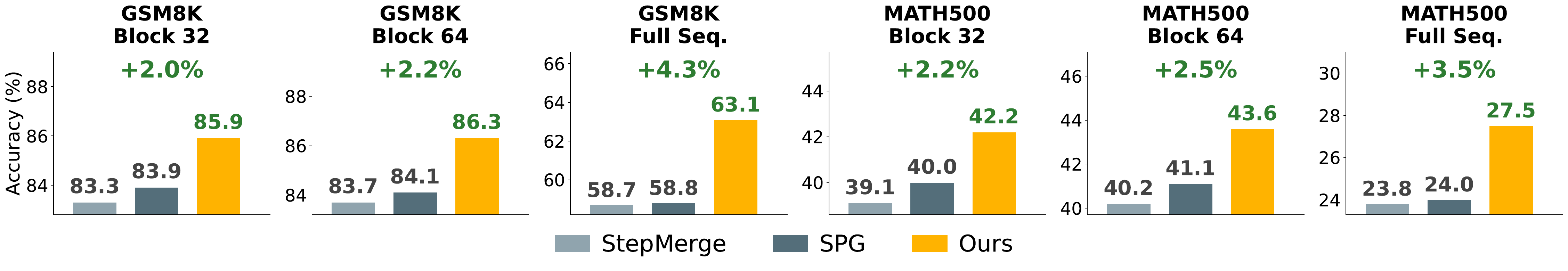}
    \caption{Ablation over inference block size. All methods use confidence-based unmasking at generation length 256. Results for baselines are from SPG~\citep{spg}.}
    \label{fig:blocksize_ablation}
\end{figure}

\section{Related work}
\label{sec:related_work}

\paragraph{Diffusion Language Models.}
Discrete diffusion models for text generation have developed along several lines.
Early work applied diffusion in continuous embedding space~\citep{diffusionlm, diffuseq}, while more recent approaches operate directly in discrete token space.
Score Entropy Discrete Diffusion (SEDD)~\citep{sedd} and Masked Diffusion Language Models (MDLM)~\citep{mdlm, shi2024simplified, ou2024absorbing} established masked diffusion as a competitive paradigm for text generation, with simplified training objectives based on the ELBO.
Scaling these models to larger sizes, LLaDA~\citep{llada} and Dream~\citep{dream} demonstrated that masked diffusion LLMs can approach autoregressive performance on standard benchmarks.
More recently, diffusion language models have been extended to multimodal settings~\citep{mmada, lavida, fudoki, unidisc}, and several works have improved inference efficiency through caching~\citep{dkvcache, dllmcache} and parallel decoding strategies~\citep{fastdllm, block_diffusion}.

\paragraph{Reinforcement Learning for Language Models.}
Reinforcement learning has become a central technique for post-training large language models, building on classical policy gradient methods~\citep{reinforce, policy_gradient}.
Early work used RL from human feedback (RLHF) with PPO~\citep{ppo} and TRPO~\citep{trpo} to align models with human preferences~\citep{christiano2017deep, rlhf}.
More recently, RL has been applied to improve reasoning capabilities.
DeepSeekMath~\citep{shao2024deepseekmath} introduced Group Relative Policy Optimization (GRPO), a critic-free policy gradient method that estimates baselines from group rollouts, removing the need for a value network.
DeepSeek-R1~\citep{deepseekr1} and Kimi K1.5~\citep{team2025kimi} scaled RL-based reasoning training, demonstrating that policy gradient methods with verifiable rewards can elicit strong mathematical and coding performance.
Subsequent work has refined the GRPO objective, reducing bias~\citep{drgrpo}, improving sample efficiency~\citep{spo}, and introducing sequence-level clipping~\citep{gspo}.

\paragraph{Reinforcement Learning for Diffusion Language Models.}
Reinforcement learning for MDLMs has been used for reward optimization~\citep{drakes}, preference alignment~\citep{llada1.5}, faster parallel decoding~\citep{yang2025taming, jazbec2025learning, chen2025dultra, zhao2025diffpo}, and reasoning~\citep{d1, gdpo, spg, d2}.
For reasoning, the central challenge of RL is estimating the intractable sequence log-likelihood.
ELBO-based methods estimate log-probabilities via variants of the training objective: d1~\citep{d1} uses a mean-field approximation, UniGRPO~\citep{mmada} and ESPO~\citep{ou2025espo} use sequence-level Monte Carlo estimators, GDPO~\citep{gdpo} tightens the bound with Gaussian quadrature, wd1~\citep{tang2025wd1} uses weighted optimization, and SPG~\citep{spg} combines ELBO and EUBO bounds with block-wise masking.
Trajectory-based methods instead decompose the log-likelihood over denoising steps, attempting to follow the generation trajectory: d2~\citep{d2} and \citet{wang2025revolutionizing} partition the trajectory into segments via the StepMerge estimator.
However, these trajectory-based methods model the trajectory as a sequence of token prediction steps only, and in doing so ignore the position selection decision that MDLMs make at generation time.
Our work addresses this gap.

\paragraph{Learning the unmasking order.}
A line of work aims to learn or optimize the unmasking order of MDLMs, frequently to accelerate parallel decoding by reducing the number of forward passes.
Jazbec et al.~\citep{jazbec2025learning} learn an unmasking policy over the model's confidences, dUltra~\citep{chen2025dultra} trains an unmasking planner head with reinforcement learning, and P2~\citep{peng2025path} frames sampling as path planning to choose which positions to update; Where-to-Unmask~\citep{asano2026unmask} instead trains a separate planner in a supervised fashion to imitate an oracle order derived from ground-truth tokens.
Unlike these methods, our goal is not to introduce an explicit model of the unmasking order; rather, we show that a position-selection term arises directly from the complete policy-gradient estimator, and we optimize it using the model's own logits without any additional module.
The most closely related method in this respect is LLaDOU/DCoLT~\citep{dcolt}, which also incorporates position selection into the RL policy through a Plackett--Luce distribution; unlike DCoLT, which trains a separate module to model the position logits, we reuse the base model's own logits.
We provide a detailed comparison with DCoLT in Appendix~\ref{appendix:lladou}.

\section{Conclusion}
\label{sec:conclusion}

We propose a policy gradient method for masked diffusion language models that optimizes both token prediction and position selection.
By replacing greedy top-$K$ remasking with a probabilistic variant, we obtain a differentiable distribution over unmasking positions, yielding a gradient decomposition into a token term and a masking term.
Extensive experiments across mathematical reasoning and code generation benchmarks show that jointly optimizing these terms in the policy gradient achieves significant improvements over existing methods and state-of-the-art results.

\section*{Acknowledgments}
The experiments in this work were run on the Vista GPU Cluster through the Center for Generative AI (CGAI) and the Texas Advanced Computing Center (TACC) at The University of Texas at Austin.

\bibliography{colm2026_conference}
\bibliographystyle{colm2026_conference}

\appendix

\section{Training details}
\label{appendix:training_details}

\paragraph{Base model.}
We use LLaDA-8B-Instruct~\citep{llada} as the base model for all experiments.
LLaDA is a masked diffusion language model that uses absorbing-state diffusion with a special \texttt{[MASK]} token (ID 126336).
We do not apply supervised fine-tuning before RL training; all results use the pretrained instruction-tuned checkpoint directly.

\paragraph{Parameter-efficient fine-tuning.}
We use Low-Rank Adaptation (LoRA) for parameter-efficient fine-tuning.
We set the LoRA rank to $r=128$, scaling factor $\alpha=64$, and dropout $0.05$.
LoRA adapters are applied to all attention projection layers (query, key, value, and output) and all MLP projection layers (gate, up, and down).
The base model weights are frozen in 4-bit quantization (NF4 format) to reduce memory usage, with all LoRA parameters and computations in bfloat16.
We use Flash Attention 2~\citep{dao2023flashattention} for memory-efficient attention computation.

\paragraph{Optimizer.}
We use AdamW~\citep{adamw} with $\beta_1 = 0.9$, $\beta_2 = 0.99$, and weight decay $0.1$.
The learning rate is set to $3 \times 10^{-6}$ with a constant schedule and a brief warmup (warmup ratio $0.0001$).
We clip gradients to a maximum norm of $0.2$.

\paragraph{Batch size and gradient accumulation.}
We use a per-device training batch size of 3 with gradient accumulation over 2 steps.
Training is distributed across 8 GPUs, yielding an effective batch size of $3 \times 2 \times 8 = 48$ prompt-completion pairs per optimization step.
For each prompt, we sample $G=6$ completions for group-relative advantage estimation.

\paragraph{Policy optimization.}
We optimize using GSPO~\citep{gspo} with separate clipping for the token and position importance ratios.
GSPO clips the importance ratio with a lower bound of $3 \times 10^{-4}$ and an upper bound of $4 \times 10^{-4}$.
No KL regularization term is used.
We perform $\mu = 12$ inner gradient updates per batch of generated completions.
We synchronize the reference model every 64 optimization steps.

\paragraph{Generation parameters.}
During RL rollouts, we generate completions with a maximum length of 256 tokens.
Generation uses semi-autoregressive block-wise decoding with a block size of 32 tokens and 128 diffusion steps per block.
At each denoising step, we use probabilistic remasking.
The sampling temperature is set to 0.9 for token selection and 0.5 for remasking.

\paragraph{Datasets.}
For mathematical reasoning, we train on the training split of GSM8K~\citep{gsm8k} for grade-school math and on the training split of the MATH dataset for competition-level math, evaluating on the MATH500~\citep{math500} test set.
We follow the same train-test splitting as d1~\citep{d1} and SPG~\citep{spg}.
For code generation, we train on KodCode-Light-RL-10K~\citep{kodcode}, a dataset of coding problems at varying difficulty levels with synthetic unit tests for reward computation.
We evaluate on HumanEval~\citep{humaneval} (164 hand-crafted Python problems) and sanitized MBPP~\citep{mbpp} (257 crowd-sourced Python tasks).
We follow the same dataset setup as GDPO~\citep{gdpo}.

\paragraph{Reward functions.}
We follow the reward functions of d1~\citep{d1} and SPG~\citep{spg} for mathematical reasoning, and GDPO~\citep{gdpo} for code generation.

\subparagraph{GSM8K.} We use a composite reward with five additive components:
\begin{itemize}
    \item XML Structure Reward: $+0.125$ per correctly placed opening or closing tag.
    \item Soft Format Reward: $+0.5$ if the output matches the pattern \texttt{<reasoning>...</reasoning><answer>...</answer>}.
    \item Strict Format Reward: $+0.5$ for exact formatting with correct line breaks.
    \item Integer Answer Reward: $+0.5$ if the extracted answer is a valid integer.
    \item Correctness Reward: $+2.0$ if the extracted answer matches the ground truth.
\end{itemize}

\subparagraph{MATH500.} We use a two-component reward:
\begin{itemize}
    \item Format Reward: $+1.0$ if answer tags are present with \texttt{\textbackslash boxed\{\}} inside; $+0.75$ if answer tags are present without \texttt{\textbackslash boxed\{\}}; $+0.5$ if \texttt{\textbackslash boxed\{\}} is present without answer tags; $+0.25$ otherwise.
    \item Correctness Reward: $+2.0$ if the answer in \texttt{\textbackslash boxed\{\}} matches the ground truth.
\end{itemize}

\subparagraph{Code generation (HumanEval, MBPP).} We use a two-component reward:
\begin{itemize}
    \item Format Reward: $+1.0$ if the output contains a valid Python code block with no syntax errors; $+0.5$ if well-formatted but with syntax errors; $+0.0$ otherwise.
    \item Code Execution Reward: the fraction of unit tests passed. We block unsafe modules (os, sys, shutil, subprocess, socket, psutil, ctypes, pathlib, builtins, \_\_import\_\_) and assign a reward of $0$ if any are used.
\end{itemize}

\paragraph{Training schedule.}
We train each dataset independently.
For GSM8K, we train for approximately 8,000 optimization steps.
For MATH500, we train for approximately 8,000 steps.
For code generation (KodCode), we train for approximately 6,000 steps.
We save checkpoints every 100 steps and evaluate at each checkpoint.

All experiments are conducted on 8 NVIDIA H100-80GB GPUs with DeepSpeed ZeRO-2.



\begin{figure}[h]
\centering
\begin{subfigure}[t]{0.32\textwidth}
    \centering
    \includegraphics[width=\linewidth]{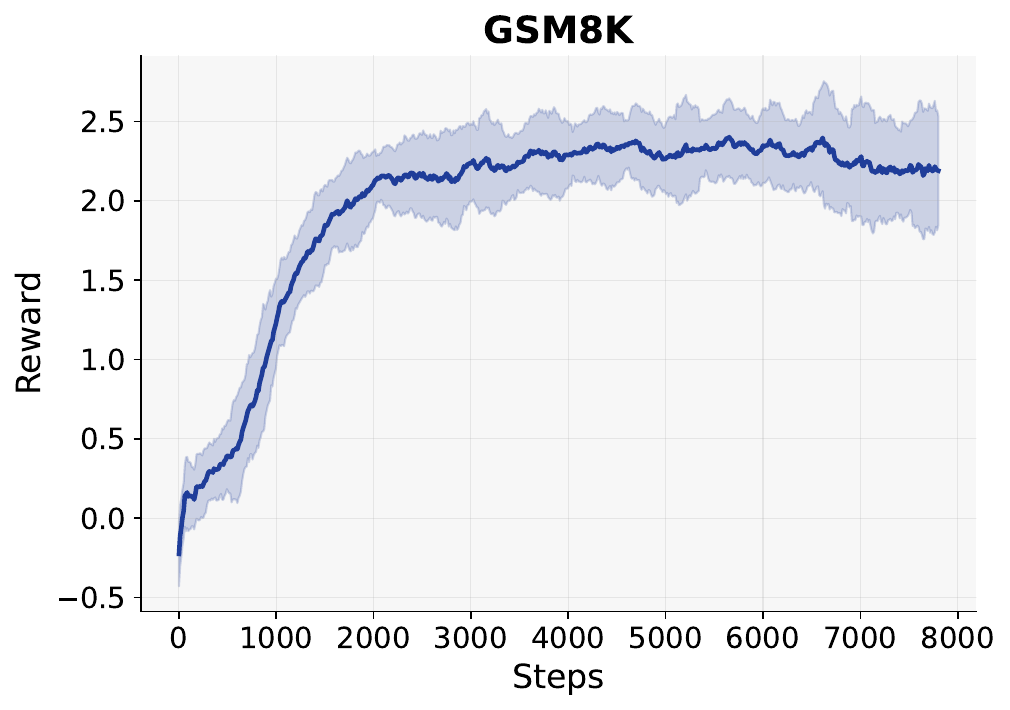}
\end{subfigure}
\hfill
\begin{subfigure}[t]{0.32\textwidth}
    \centering
    \includegraphics[width=\linewidth]{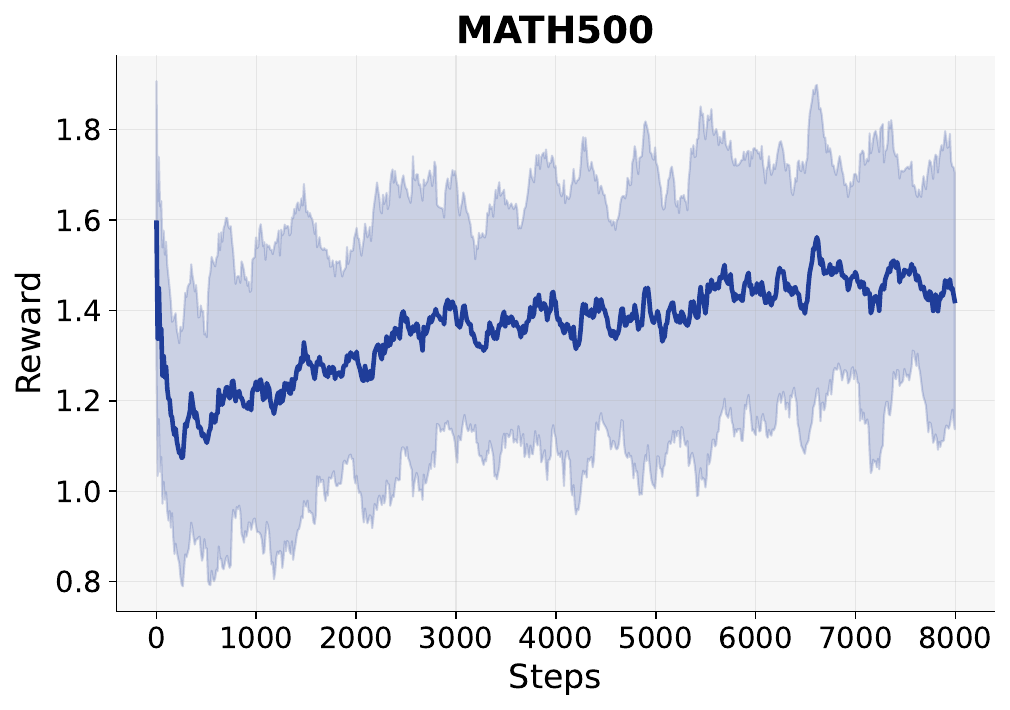}
\end{subfigure}
\hfill
\begin{subfigure}[t]{0.32\textwidth}
    \centering
    \includegraphics[width=\linewidth]{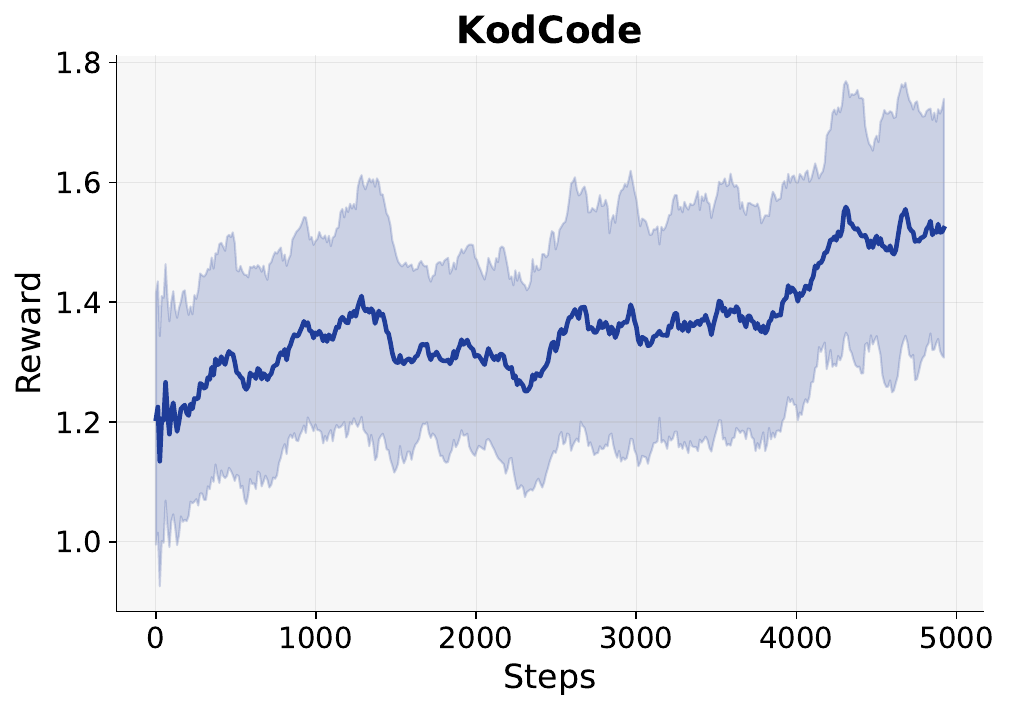}
\end{subfigure}
\caption{Training reward dynamics for our method on GSM8K, MATH500, and KodCode. Curves show EMA-smoothed reward with one standard deviation band.}
\label{fig:reward_curves}
\end{figure}

\section{Results on planning tasks}
\label{appendix:planning}

We additionally evaluate on two planning benchmarks, Sudoku~\citep{sudoku} and Countdown~\citep{countdown}, following the evaluation protocol of SPG~\citep{spg} and d1~\citep{d1}. Following SPG, Sudoku uses 3 few-shot examples during rollouts and evaluation. Table~\ref{tab:planning_results} reports accuracy at generation lengths 128, 256, and 512; results for prior methods are taken from SPG~\citep{spg}. Our method achieves the best results on both tasks at every generation length, outperforming the strongest baseline, SPG, with the largest gain on Sudoku (+3.1 at length 512). This is consistent with our finding that the benefit of position-aware optimization grows as the number of masked positions increases.

\begin{table}[H]
\centering
\begin{tabular}{lcccccc}
\toprule
& \multicolumn{3}{c}{\textbf{Sudoku}} & \multicolumn{3}{c}{\textbf{Countdown}} \\
\cmidrule(lr){2-4} \cmidrule(lr){5-7}
\textbf{Model / Seq Len} & \textbf{128} & \textbf{256} & \textbf{512} & \textbf{128} & \textbf{256} & \textbf{512} \\
\midrule
LLaDA-8B-Inst. & 5.7  & 27.7 & 26.2 & 18.8 & 16.8 & 16.8 \\
LLaDA-1.5      & 7.4  & 26.9 & 29.0 & 21.9 & 21.1 & 21.5 \\
d1             & 7.2  & 32.5 & 29.3 & 30.9 & 30.9 & 34.4 \\
wd1            & 33.1 & 32.1 & 22.5 & 48.8 & 52.3 & 50.8 \\
UniGRPO        & 59.0 & 67.0 & 62.9 & 44.5 & 43.0 & 57.0 \\
SPG w/ Mixture & \underline{82.9} & \underline{94.0} & \underline{93.1} & \underline{68.8} & \underline{70.7} & \underline{70.3} \\
\midrule
Ours           & \textbf{84.5} & \textbf{94.9} & \textbf{96.2} & \textbf{69.3} & \textbf{71.4} & \textbf{71.0} \\
\bottomrule
\end{tabular}
\caption{Results on two planning benchmarks, Sudoku and Countdown, using LLaDA-8B-Instruct as the base model. Results for prior methods are taken from SPG~\citep{spg}; following SPG, Sudoku uses 3 few-shot examples during rollouts. Best results are in \textbf{bold}, second best are \underline{underlined}.}
\label{tab:planning_results}
\end{table}

\section{Generalization to Dream-7B}
\label{appendix:dream}

We additionally evaluate our method on a different base model, Dream-7B~\citep{dream}. Table~\ref{tab:dream} reports GSM8K, MATH500, and Sudoku accuracy at generation lengths 128, 256, and 512; SPG results are obtained by running their publicly available code. Our method outperforms all baselines across all three benchmarks and generation lengths, with improvements of up to $2.9$ points over the strongest baseline.

\begin{table}[H]
\centering
\adjustbox{width=\textwidth,center}{
\begin{tabular}{lccccccccc}
\toprule
& \multicolumn{3}{c}{\textbf{GSM8K}} & \multicolumn{3}{c}{\textbf{MATH500}} & \multicolumn{3}{c}{\textbf{Sudoku}} \\
\cmidrule(lr){2-4} \cmidrule(lr){5-7} \cmidrule(lr){8-10}
\textbf{Model / Seq Len} & \textbf{128} & \textbf{256} & \textbf{512} & \textbf{128} & \textbf{256} & \textbf{512} & \textbf{128} & \textbf{256} & \textbf{512} \\
\midrule
Dream-7B       & 75.8 & 81.3 & 80.7 & 38.2 & 45.7 & 48.0 & 9.3 & 2.1 & 14.0 \\
d1             & 77.0 & 81.9 & 81.7 & 39.4 & 46.9 & 48.9 & 64.4 & 69.7 & 51.1 \\
wd1            & 76.3 & 82.4 & 82.9 & 39.5 & 47.4 & 50.5 & 29.5 & 39.2 & 30.3 \\
ESPO           & 79.6 & 82.3 & 82.0 & 40.3 & 47.4 & 50.3 & 71.7 & 72.3 & 71.3 \\
SPG w/ Mixture & \underline{80.1} & \underline{82.6} & \underline{83.5} & \underline{41.1} & \underline{48.5} & \underline{50.7} & \underline{72.1} & \underline{72.4} & \underline{71.9} \\
\midrule
Ours           & \textbf{82.5} & \textbf{84.1} & \textbf{86.4} & \textbf{43.8} & \textbf{50.3} & \textbf{52.9} & \textbf{73.4} & \textbf{74.1} & \textbf{74.3} \\
\bottomrule
\end{tabular}
}
\caption{Generalization to Dream-7B~\citep{dream} on GSM8K, MATH500, and Sudoku, at generation lengths 128, 256, and 512. SPG results are obtained by running their publicly available code. Best results in \textbf{bold}, second best \underline{underlined}.}
\label{tab:dream}
\end{table}

\section{Comparison to LLaDOU}
\label{appendix:lladou}

LLaDOU~\citep{dcolt}, trained with the Diffusion Chain of Lateral Thought (DCoLT) algorithm, shares our high-level insight that the unmasking order should be part of the policy, and arrives at a similar probabilistic framework based on Plackett--Luce position sampling. The key difference is architectural: LLaDOU introduces a separately trained head for position selection, whereas our method derives position scores directly from the model's own token log-likelihoods, requiring no additional parameters or architectural changes. This has the practical advantage that our approach integrates directly with existing MDLM training infrastructure such as StepMerge~\citep{d2}, and applies to any off-the-shelf MDLM.

We compare the two approaches on GSM8K in Table~\ref{tab:lladou}, reporting accuracy, training throughput, and total training cost. When DCoLT uses full trajectory back-propagation (full replay), it reaches accuracy comparable to ours but at roughly $6\times$ lower throughput and $5\times$ higher total GPU cost. In a matched comparison where both methods use the StepMerge approximation, our method outperforms DCoLT by approximately $5\%$ on average across generation lengths. This efficiency gain does not come solely from StepMerge: even with StepMerge, DCoLT still requires about $5\times$ more GPU hours to train ($\sim$800 vs $\sim$160), because it must train a separate position-selection head from scratch. Our method instead derives position scores from the model's existing token log-likelihoods, providing a useful signal from the start of training and converging much faster. Overall, our method matches or exceeds DCoLT's accuracy while being substantially cheaper to train and applicable to any MDLM without architectural changes.

\begin{table}[H]
\centering
\begin{tabular}{lccccc}
\toprule
\textbf{Method} & \textbf{128} & \textbf{256} & \textbf{512} & \textbf{Throughput (steps/h)} & \textbf{GPU hours} \\
\midrule
DCoLT (full replay)  & 80.4 & 84.3 & 86.2 & 48  & $\sim$800 \\
DCoLT (w/ StepMerge) & 77.1 & 80.1 & 82.3 & 280 & $\sim$800 \\
\midrule
Ours                 & \textbf{81.0} & \textbf{85.9} & \textbf{87.1} & \textbf{290} & \textbf{$\sim$160} \\
\bottomrule
\end{tabular}
\caption{Comparison with LLaDOU (trained with DCoLT) on GSM8K. We report accuracy at generation lengths 128, 256, and 512, training throughput, and total training cost (number of GPUs $\times$ hours). Our method matches or exceeds LLaDOU's accuracy at a fraction of the training cost. Best accuracy in \textbf{bold}.}
\label{tab:lladou}
\end{table}

\section{Additional ablations}
\label{appendix:ablations}

\subsection{Inference block size}
\label{appendix:blocksize_ablation}

Table~\ref{tab:blocksize_ablation} reports the full numerical results for the inference block size ablation discussed in Section~\ref{sec:results} and summarized in Figure~\ref{fig:blocksize_ablation}.

\begin{table}[h]
\centering
\adjustbox{width=\textwidth,center}{
\begin{tabular}{lcccccc}
\toprule
& \multicolumn{2}{c}{\textbf{Block 32}} & \multicolumn{2}{c}{\textbf{Block 64}} & \multicolumn{2}{c}{\textbf{Full Sequence}} \\
\cmidrule(lr){2-3} \cmidrule(lr){4-5} \cmidrule(lr){6-7}
\textbf{Model} & \textbf{GSM8K} & \textbf{MATH500} & \textbf{GSM8K} & \textbf{MATH500} & \textbf{GSM8K} & \textbf{MATH500} \\
\midrule
LLaDA-8B-Inst. & 77.2 & 32.4 & 78.6 & 33.2 & 23.9 & 17.8 \\
LLaDA-1.5     & 80.5 & 32.2 & 81.0 & 35.4 & 41.4 & 20.4 \\
d1            & 80.6 & 36.0 & 80.9 & 37.6 & 57.5 & 22.6 \\
wd1           & 81.5 & 37.4 & 82.5 & 37.4 & 56.7 & \underline{25.0} \\
UniGRPO       & 82.5 & 37.4 & 82.3 & 37.4 & 50.0 & 24.2 \\
SPG w/ Mixture & \underline{83.9} & \underline{40.0} & \underline{84.1} & \underline{41.1} & \underline{58.8} & 24.0 \\
StepMerge     & 83.3 & 39.1 & 83.7 & 40.2 & 58.7 & 23.8 \\
\midrule
Ours          & \textbf{85.9} & \textbf{42.2} & \textbf{86.3} & \textbf{43.6} & \textbf{63.1} & \textbf{27.5} \\
\bottomrule
\end{tabular}
}
\caption{Ablation over inference block size. All methods use confidence-based unmasking at generation length 256. Results for baselines are from SPG~\citep{spg}.}
\label{tab:blocksize_ablation}
\end{table}

\subsection{RL algorithm}

Our decomposition is compatible with any policy gradient algorithm.
We compare three widely used algorithms, RLOO~\citep{Kool2019Buy4R}, GRPO, and GSPO, in Table~\ref{tab:rl_ablation} with $N=32$ and $K=12$, evaluating on GSM8K and MATH500 at generation length 128.
GSPO performs the strongest on both benchmarks, which we attribute to its sequence-level clipping being better suited to handle the variance introduced by subsampling $K$ boundaries.

\begin{table}[t!]
\centering
\begin{subtable}[t]{0.36\textwidth}
\centering
\begin{tabular}{lcc}
\toprule
\textbf{Algorithm} & \textbf{GSM8K} & \textbf{MATH500} \\
\midrule
RLOO & 75.0 & 30.7 \\
GRPO      & 79.5 & 36.4 \\
\underline{GSPO}      & \textbf{81.0} & \textbf{37.4} \\
\bottomrule
\end{tabular}
\caption{RL algorithm.}
\label{tab:rl_ablation}
\end{subtable}
\hfill
\begin{subtable}[t]{0.56\textwidth}
\centering
\begin{tabular}{cccc}
\toprule
\textbf{N} & \textbf{K} & \textbf{GSM8K} & \textbf{MATH500} \\
\midrule
\multirow{2}{*}{16} & 8  & 79.9 & 36.8 \\
                     & 16 & 80.6 & 37.2 \\
\midrule
\multirow{5}{*}{\underline{32}} & 1  & 78.0 & 35.4 \\
                     & 4  & 79.4 & 36.3 \\
                     & 8  & 80.3 & 36.9 \\
                     & \underline{12} & 81.0 & \textbf{37.4} \\
                     & 16 & \textbf{81.3} & \textbf{37.4} \\
\midrule
\multirow{2}{*}{64} & 8  & 79.8 & 36.8 \\
                     & 16 & 80.2 & 37.0 \\
\bottomrule
\end{tabular}
\caption{StepMerge configuration: $N$ segments, $K$ sampled.}
\label{tab:nk_ablation}
\end{subtable}
\caption{Ablations on (a) RL algorithm and (b) StepMerge configuration. All use position-aware trajectory likelihood, evaluated on GSM8K and MATH500 at generation length 128. \underline{Underlined} entries denote the configuration used in our main results. Best results in \textbf{bold}.}
\end{table}

\subsection{StepMerge configuration}

Table~\ref{tab:nk_ablation} ablates the number of segments $N$ and subsampled boundaries $K$, evaluating on GSM8K and MATH500 at generation length 128.
Performance improves with larger $N$, as finer segmentation reduces the approximation error.
Performance degrades only modestly as $K$ decreases (at $N=32$, from 81.0 at $K=12$ to 78.0 at $K=1$ on GSM8K), indicating the estimator remains effective even with few sampled boundaries.
However, increasing $K$ beyond 12 provides only marginal gains: at $N=32$, $K=12$ matches $K=16$ on MATH500 while requiring fewer forward passes.
We use $N=32$, $K=12$ for our main results.

\subsection{Sampling temperature}
\label{appendix:temperature_ablation}

Our method uses two temperatures: a token sampling temperature $\tau_{\mathrm{tok}}$ and a position selection temperature $\tau_{\mathrm{pos}}$ from our probabilistic position selection (Eq.~\ref{eq:pos_prob}), set to $\tau_{\mathrm{tok}}=0.9$ (matching prior work) and $\tau_{\mathrm{pos}}=0.5$ by default.
We ablate both in Table~\ref{tab:temp_ablation} on GSM8K and MATH500 at generation length 128.
Performance degrades as $\tau_{\mathrm{tok}}$ is lowered, since less diverse rollouts weaken the group-relative advantage signal, and as $\tau_{\mathrm{pos}}$ is raised, since more random selection unmasks low-confidence positions.
The defaults strike a good balance, with results remaining reasonably stable across the tested range.

\begin{table}[H]
\centering
\begin{subtable}[t]{0.46\textwidth}
\centering
\begin{tabular}{lcc}
\toprule
$\bm{\tau_{\mathrm{tok}}}$ & \textbf{GSM8K} & \textbf{MATH500} \\
\midrule
0.5            & 79.1          & 35.9 \\
\underline{0.9} & \textbf{81.0} & \textbf{37.4} \\
1.5            & 80.5          & 37.0 \\
\bottomrule
\end{tabular}
\caption{Token sampling temperature $\tau_{\mathrm{tok}}$ (with $\tau_{\mathrm{pos}}=0.5$).}
\label{tab:temp_tok}
\end{subtable}
\hfill
\begin{subtable}[t]{0.46\textwidth}
\centering
\begin{tabular}{lcc}
\toprule
$\bm{\tau_{\mathrm{pos}}}$ & \textbf{GSM8K} & \textbf{MATH500} \\
\midrule
0.1            & 80.7          & 37.3 \\
\underline{0.5} & \textbf{81.0} & \textbf{37.4} \\
1.0            & 79.0          & 35.7 \\
\bottomrule
\end{tabular}
\caption{Position selection temperature $\tau_{\mathrm{pos}}$ (with $\tau_{\mathrm{tok}}=0.9$).}
\label{tab:temp_pos}
\end{subtable}
\caption{Ablation over the two sampling temperatures, evaluated on GSM8K and MATH500 at generation length 128. \underline{Underlined} entries denote the configuration used in our main results. Best results in \textbf{bold}.}
\label{tab:temp_ablation}
\end{table}

\subsection{Unmasking-set size}
\label{appendix:ut_ablation}

At each denoising step, our method selects a set $\mathcal{U}_t$ of positions to unmask; for all main results we use $|\mathcal{U}_t|=2$ per step, matching d1~\citep{d1}, wd1~\citep{wd1}, and SPG~\citep{spg}.
The selection distribution $p_{\mathrm{unmask}}(\mathcal{U}_t)$ in Eq.~\ref{eq:pos_prob} is defined for any $|\mathcal{U}_t|$, and Table~\ref{tab:ut_ablation} varies it from 2 to 8 on GSM8K with all other settings fixed.
Our method remains robust across this range, degrading by at most 1.0 point from $|\mathcal{U}_t|=2$ to $8$.
This modest drop is expected, since larger sets commit more positions per step and yield lower-quality rollouts during training.

\begin{table}[H]
\centering
\begin{tabular}{lccc}
\toprule
$\bm{|\mathcal{U}_t|}$ & \textbf{Gen Len 128} & \textbf{Gen Len 256} & \textbf{Gen Len 512} \\
\midrule
\underline{2} & \textbf{81.0} & \textbf{85.9} & \textbf{87.1} \\
4             & 80.4          & 85.1          & 86.8 \\
8             & 80.1          & 84.9          & 86.5 \\
\bottomrule
\end{tabular}
\caption{Ablation over the unmasking-set size $|\mathcal{U}_t|$ on GSM8K. We unmask $|\mathcal{U}_t|=2$ positions per step in all main results, matching the configuration of d1, wd1, and SPG. \underline{Underlined} entries denote the configuration used in our main results. Best results in \textbf{bold}.}
\label{tab:ut_ablation}
\end{table}

\section{StepMerge approximation}
\label{appendix:stepmerge_approx}

In our experiments, we approximate the token component of the policy gradient objective using the conditional distribution over each unmasked token, $\pi_{\theta}(\vecz_i^t \mid \vecc, \vecz^{t-1})$. Specifically, following \citet{d2}, we approximate the log-likelihood term in the token gradient by $\sum_{i \in \mathcal U_t} \log \pi_{\theta}(\vecz_i^t \mid \vecc, \vecz^{t-1})$.

Computing trajectory log-probability in the policy gradient using above approximation requires evaluating all $T$ denoising transitions, which is computationally expensive. Following~\citet{d2}, we adopt the StepMerge approximation, which partitions the trajectory into $N$ segments with $N \ll T$. Let ${t_0=0 < t_1 < \cdots < t_N = T}$ denote segment boundaries. Instead of evaluating every transition, we approximate the trajectory likelihood using only the segment endpoints $\log \pi_{\theta}(\vecz^{t_i} \mid \vecz^{t_{i-1}}, \vecc)$. 
Intuitively, this approximation treats multiple denoising steps as a single macro-transition, reducing the number of required forward passes from $T$ to $N$ while preserving the overall trajectory structure.


\subsection{StepMerge approximation error}
\label{appendix:stepmerge_approx_error}

We analyze the approximation error induced by using StepMerge to evaluate the trajectory
log-probability. The error has three sources: within-segment unmasking timing information lost by
the macro-stage approximation, token-probability error from using the segment boundary context,
and unmasking-probability error from using the same boundary context.

\begin{proposition}[Macro-StepMerge approximation error]
\label{thm:macro_stepmerge_error}
Let the $T$ denoising steps be partitioned into $N$ equal-length StepMerge segments with
boundaries $0=t_0<t_1<\cdots<t_N=T$. Let $I_m=\{t_{m-1}+1,\ldots,t_m\}$ be segment $m$. 

Suppose that, for every segment $m$, every $t\in I_m$, and every retained position
$i\in\mathcal U_t$,
\begin{equation}
    \log
    \frac{
    \pi_{\bm{\theta}}(z_i^t\mid \vecz^{t-1},\vecc)
    }{
    \pi_{\bm{\theta}}(z_i^t\mid \vecz^{t_{m-1}},\vecc)
    }
    \leq
    \epsilon_{\mathrm{tok}}.
\label{eq:macro_tok_stability}
\end{equation}
Suppose also that the unmasking distribution is stable within each segment:
\begin{equation}
    \log
    \frac{
    p_{\mathrm{unmask}}(\mathcal U_t\mid \vecz^{t-1},\vecc)
    }{
    p_{\mathrm{unmask}}(\mathcal U_t\mid \vecz^{t_{m-1}},\vecc)
    }
    \leq
    |\mathcal U_t|\epsilon_{\mathrm{unmask}}.
\label{eq:macro_unmask_stability}
\end{equation}
Then the macro-StepMerge approximation error satisfies
\begin{equation}
    D_N
    \leq
    n\log\!\left(\frac{T}{N}+1\right)
    +
    n\epsilon_{\mathrm{tok}}
    +
    n\epsilon_{\mathrm{unmask}}.
\label{eq:macro_stepmerge_bound}
\end{equation}
\end{proposition}

\begin{proof}
Let the full length-$T$ trajectory be
$\xi = (\vecz^0,\hat{\vecz}^1,\mathcal U_1,\vecz^1,\ldots,\vecz^{T-1},\hat{\vecz}^T,\mathcal U_T,\vecz^T).$
For segment $m$, define the segment-level unmasking set
$\bar{\mathcal U}^{\,m} = \bigcup_{t\in I_m}\mathcal U_t.$
The macro-StepMerge block trajectory for the segment keeps only
$\xi_{\mathrm{SM}}^m = \left(\vecz^{t_{m-1}},\bar{\mathcal U}^{\,m},\vecz^{t_m}\right).$
It records the start state, the end state, and the positions unmasked somewhere in the segment,
but not the exact within-segment step at which those positions were unmasked.

Let the skipped intermediate information be
$ H^m = \left(
    \vecz^{t_{m-1}+1},\hat{\vecz}^{t_{m-1}+1},
    \ldots,
    \vecz^{t_m-1},\hat{\vecz}^{t_m}
    \right),$ and let the within-segment timing variable be
    $\Lambda^m = \left(
    \mathcal U_{t_{m-1}+1},
    \mathcal U_{t_{m-1}+2},
    \ldots,
    \mathcal U_{t_m}
    \right).$
Thus the full block can be represented as $(\xi_{\mathrm{SM}}^m,H^m,\Lambda^m)$.

Let $\mathcal F_{m-1}$ denote the history up to the start of block $m$. By the chain rule of probability, the exact block
conditional distribution is
\begin{align}
    p_m(\xi_{\mathrm{SM}}^m,H^m,\Lambda^m\mid \mathcal F_{m-1})
    &=
    p_m(\xi_{\mathrm{SM}}^m\mid \mathcal F_{m-1})
    p_m(H^m\mid \xi_{\mathrm{SM}}^m,\mathcal F_{m-1}) \notag\\
    &\quad \cdot
    p_m(\Lambda^m\mid H^m,\xi_{\mathrm{SM}}^m,\mathcal F_{m-1}).
\label{eq:exact_block_factor}
\end{align}

The StepMerge block distribution on the same variables is given by:
\begin{align}
    q_m(\xi_{\mathrm{SM}}^m,H^m,\Lambda^m\mid \mathcal F_{m-1})
    &=
    q_m(\xi_{\mathrm{SM}}^m\mid \mathcal F_{m-1})
    p_m(H^m\mid \xi_{\mathrm{SM}}^m,\mathcal F_{m-1}) \notag\\
    &\quad \cdot
    p_m(\Lambda^m\mid \xi_{\mathrm{SM}}^m,\mathcal F_{m-1}).
\label{eq:sm_block_factor}
\end{align}
The timing distribution in Eq.~\ref{eq:sm_block_factor}
does not condition on $H^m$, because StepMerge only knows the coarsened block trajectory
$\xi_{\mathrm{SM}}^m$ when assigning within-segment timing and does not depend on $H^m$.

The full and StepMerge approximation trajectory distributions factor over blocks:
\begin{equation}
    p(\xi)
    =
    \prod_{m=1}^N
    p_m(\xi_{\mathrm{SM}}^m,H^m,\Lambda^m\mid \mathcal F_{m-1}),
    \qquad
    q(\xi)
    =
    \prod_{m=1}^N
    q_m(\xi_{\mathrm{SM}}^m,H^m,\Lambda^m\mid \mathcal F_{m-1}).
\end{equation}
Therefore, the KL divergence between full and StepMerge trajectories is
\begin{align}
    D_N
    &:=
    D_{\mathrm{KL}}(p(\xi)\,\|\,q(\xi)) \notag\\
    &=
    \sum_{m=1}^N
    \mathbb E_{\mathcal F_{m-1}\sim p}
    \left[
    D_{\mathrm{KL}}\!\left(
    p_m(\cdot\mid\mathcal F_{m-1})
    \,\middle\|\,
    q_m(\cdot\mid\mathcal F_{m-1})
    \right)
    \right].
\label{eq:block_kl_decomp}
\end{align}

Using Eqs.~\ref{eq:exact_block_factor} and~\ref{eq:sm_block_factor}, the $H^m$ terms cancel in
the block log-ratio:
\begin{align}
    \log \frac{p_m}{q_m}
    &=
    \log
    \frac{
    p_m(\Lambda^m\mid H^m,\xi_{\mathrm{SM}}^m,\mathcal F_{m-1})
    }{
    p_m(\Lambda^m\mid \xi_{\mathrm{SM}}^m,\mathcal F_{m-1})
    } +
    \log
    \frac{
    p_m(\xi_{\mathrm{SM}}^m\mid \mathcal F_{m-1})
    }{
    q_m(\xi_{\mathrm{SM}}^m\mid \mathcal F_{m-1})
    }.
\end{align}
Taking expectations gives
\begin{equation}
    D_N
    =
    D_{\mathrm{time}}
    +
    D_{\mathrm{coarse}},
\end{equation}
where
\begin{equation}
    D_{\mathrm{time}}
    =
    \sum_{m=1}^N
    I\!\left(
    \Lambda^m;H^m
    \mid
    \xi_{\mathrm{SM}}^m,\mathcal F_{m-1}
    \right),
\label{eq:dtime_cmi}
\end{equation}
and $D_{\mathrm{coarse}}$ is the remaining KL between exact and StepMerge macro-block
likelihoods. We first bound the timing term. By the conditional mutual information identity,
\begin{align*}
    I(\Lambda^m;H^m\mid \xi_{\mathrm{SM}}^m,\mathcal F_{m-1})
    & =
    H(\Lambda^m\mid \xi_{\mathrm{SM}}^m,\mathcal F_{m-1})
    -
    H(\Lambda^m\mid H^m,\xi_{\mathrm{SM}}^m,\mathcal F_{m-1}). \\
    & \leq
    H(\Lambda^m\mid \xi_{\mathrm{SM}}^m,\mathcal F_{m-1}),
\end{align*}
where the inequality follows from the nonnegativity of conditional entropy. For each $i\in \bar{\mathcal U}^{\,m}$, let $R_i^m$ denote the exact within-segment unmasking
time of position $i$. Since $|I_m|=T/N$, there are at most $T/N+1$ possible timing choices under
the macro-stage endpoint convention. Therefore,
\begin{equation}
    H(R_i^m\mid \xi_{\mathrm{SM}}^m,\mathcal F_{m-1})
    \leq
    \log\!\left(\frac{T}{N}+1\right).
\end{equation}
The variable $\Lambda^m$ is determined by the collection of per-position timing variables, so by
subadditivity of entropy,
\begin{equation}
    H(\Lambda^m\mid \xi_{\mathrm{SM}}^m,\mathcal F_{m-1})
    \leq
    |\bar{\mathcal U}^{\,m}|\log\!\left(\frac{T}{N}+1\right).
\end{equation}
Therefore,
\begin{equation}
    D_{\mathrm{time}}
    \leq
    \sum_{m=1}^N
    |\bar{\mathcal U}^{\,m}|\log\!\left(\frac{T}{N}+1\right)
    =
    n\log\!\left(\frac{T}{N}+1\right),
\label{eq:dtime_bound_macro}
\end{equation}
because every final token position is unmasked exactly once.

We now bound the coarse term. Marginalization cannot increase KL divergence so we have:
$D_{\mathrm{KL}}(p(X,Y)\|q(X,Y)) \geq D_{\mathrm{KL}}(p(X)\|q(X))$ for random variables $X$ and $Y$.
Applying the KL inequality between $\xi_{\mathrm{SM}}^m$ and $(\mathcal U_{t_{m-1}+1},\vecz^{t_{m-1}+1},\ldots,\mathcal U_{t_m})$, we have
$\xi_{\mathrm{SM}}^m$ and $Y^m$ gives
\begin{align*}
    D_{\mathrm{coarse}}
    &=
    \sum_{m=1}^N
    \E_{\mathcal F_{m-1}\sim p}
    \left[
    D_{\mathrm{KL}}\!\left(
    p_m(\xi_{\mathrm{SM}}^m\mid \mathcal F_{m-1})
    \,\middle\|\,
    q_m(\xi_{\mathrm{SM}}^m\mid \mathcal F_{m-1})
    \right)
    \right] \\
    &\leq
    \sum_{m=1}^N
    \E_{\mathcal F_{m-1}\sim p}
    \Big[
    D_{\mathrm{KL}}\!\Big(
    p_m( 
(\mathcal U_{t_{m-1}+1},\vecz^{t_{m-1}+1},\ldots,\mathcal U_{t_m},\vecz^{t_m}) \mid \mathcal F_{m-1}) \,\Big\|\, \\
    & \qquad \qquad \qquad \qquad  \qquad q_m( (\mathcal U_{t_{m-1}+1},\vecz^{t_{m-1}+1},\ldots,\mathcal U_{t_m},\vecz^{t_m}) \mid \mathcal F_{m-1})
    \Big)
    \Big].
\end{align*}

For a realized segment, the deterministic remasking factors cancel. The exact-to-StepMerge
density ratio therefore decomposes into token and unmasking likelihood ratios:
\begin{align*}
    \log \frac{p_m(\cdot \mid \mathcal F_{m-1})}{q_m(\cdot \mid \mathcal F_{m-1})}
    &=
    \sum_{t \in I_m}
    \sum_{i \in \mathcal U_t}
    \log
    \frac{
    \pi_{\bm{\theta}}(z_i^t \mid \vecz^{t-1}, \vecc)
    }{
    \pi_{\bm{\theta}}(z_i^t \mid \vecz^{t_{m-1}}, \vecc)
    } \\
    &\quad +
    \sum_{t \in I_m}
    \log
    \frac{
    p_{\mathrm{unmask}}(\mathcal U_t \mid \vecz^{t-1},\vecc)
    }{
    p_{\mathrm{unmask}}(\mathcal U_t \mid \vecz^{t_{m-1}},\vecc)
    } \\
    &\leq
    \sum_{t \in I_m}
    |\mathcal U_t|
    \left(\epsilon_{\mathrm{tok}}+\epsilon_{\mathrm{unmask}}\right),
\end{align*}
where the last inequality follows from the two stability assumptions. Since each generated token is
unmasked exactly once, $\sum_{m=1}^N \sum_{t \in I_m} |\mathcal U_t| = n$. Hence
\begin{equation}
    D_{\mathrm{coarse}}
    \leq
    n(\epsilon_{\mathrm{tok}}+\epsilon_{\mathrm{unmask}}).
\end{equation}
Combining this bound with Eq.~\ref{eq:dtime_bound_macro} gives Eq.~\ref{eq:macro_stepmerge_bound}.
\end{proof}

\section{Incompleteness of the token-only policy gradient}
\label{appendix:suboptim}

In this section, we show that optimizing only the token gradient can miss a direction that improves the reward, in the following setting. Recall that at each denoising step $t$, the model produces probability $\pi_{\theta}(\cdot | \vecz^t)$ over the full vocabulary at every position. Let $\mathcal{M}_t$ denote the set of positions that are still masked in state $\bm{z}^t$. For each masked position $k \in \mathcal{M}_t$, we define an unmasking score $v_k^{t}$ as the highest token confidence at position $k$. The generation process then selects the top positions with the highest unmasking scores to unmask, while the rest of the tokens get remasked. We define $\nabla_{\bm{\theta}} J_{| \text{pos}}$ and $\nabla_{\bm{\theta}} J_{| \text{tok}}$ to be the position part and token part of the policy gradient as follows.

\begin{align*}
    \nabla_{\bm{\theta}} J_{| \text{pos}} &= \E_{ \hat{\bm{z}} \sim \pi_{\bm{\theta}}(\cdot \mid \bm{c}) }\bigg[ R(\bm{c}, \vecz^T)\!\sum_{t=1}^{T}\!\Big(  \nabla_{\theta} \log p_\text{unmask}(\mathcal{U}_t | \hat{\bm{z}}^t, \bm{z}^{t-1}, \bm{c}) \Big) \bigg] \\
    \nabla_{\bm{\theta}} J_{| \text{tok}} &=  \E_{\hat{\bm{z}} \sim \pi_{\bm{\theta}}(\cdot \mid \bm{c})}\bigg[ R(\bm{c}, \vecz^T)\!\sum_{t=1}^{T}\!\Big(  \nabla_{\theta} \log \pi_\theta(\vecz^t \!\mid\! \vecc, \vecz^{t-1} ) \Big) \bigg]
\end{align*}

\begin{proposition}
\label{thm:linear_suboptimality_full}
There exists a finite-horizon MDLM generation problem and a linear policy class $\pi_{\bm{\theta}}(\cdot \mid \vecz^t )$ such that, when the unmasking policy is defined by selecting positions with highest token confidence, the following holds:


There exists a direction $\bm{v}$ such that the token-only policy gradient is orthogonal to this direction, and therefore, does not move parameters along the direction of $\bm{v}$:
\begin{align*}
        \langle \nabla_{\bm{\theta}} J_{|\mathrm{tok}}(\bm{\theta}), \bm{v} \rangle = 0
    \quad \text{for all } \bm{\theta}.
\end{align*}
The true objective is not stationary along this direction, i.e.,
    \[
    \exists \bm{\theta} \text{ such that }
    \langle \nabla_{\bm{\theta}} J(\bm{\theta}), \bm{v} \rangle \neq 0.
    \]
In particular, there exists $\bm{\theta}$ and $\bm{\theta}' = \bm{\theta} + \epsilon \bm{v}$ for some $\epsilon > 0$ such that
\[
J(\bm{\theta}') > J(\bm{\theta}),
\]
while the token-only gradient provides no update in the $\bm{v}$ direction at $\bm{\theta}$.

\end{proposition}


\begin{proof}
    We construct a two-position example with sequence length $n=2$ and vocabulary $\{a,b\}$. At each step, exactly one masked position is selected and filled, producing $T=3$ states and $T-1=2$ transitions. The initial state is $\bm{z}^1 = ([\texttt{MASK}], [\texttt{MASK}])$, $\vecz^2$ has one masked token, and $\vecz^3$ is the completely unmasked sequence.

    The reward is defined as
    \begin{align*}
        R(x_1,x_2) = \mathbb{I}\{x_1 = a\}.
    \end{align*}
    We consider the linear policy class on the feature vector $\phi(\vecz, k)$ for each masked position $k \in \{1, 2\}$. More specifically, the logits of the policy is given by
    \begin{align*}
        s_{\vectheta}(\vecz, k) = \vectheta^\top \phi(\vecz, k).
    \end{align*}
    The probability of each symbol in vocabulary is given by the sigmoid of the logits.
    \begin{align*}
        p(a \; | \; \vecz, k ) = \sigma( s_{\vectheta}(\vecz, k) ), \quad \text{and} \quad p(b \; | \; \vecz, k ) = 1 - \sigma( s_{\vectheta}(\vecz, k) ).
    \end{align*}
    We consider the following set of one-hot features in $\phi(\vecz, k) \in \R^3$ to represent $\vecz$ and position $k$:
    \begin{align*}
        \phi(\vecz, k) = \begin{cases}
            \vece_1 \quad & \text{if} \quad (\vecz, k) = ((\mask, \mask), \; 1)\\
            \vece_2 \quad & \text{if} \quad (\vecz, k) = ((\mask, \mask), \; 2)\\
            \vece_3 \quad & \text{if} \quad (\vecz, k) = ((\mask, \; a), \; 1) \quad \text{or} \quad (\vecz, k) = ((\mask, \; b), \; 1) \\
        \end{cases}
    \end{align*}
    This gives us that
    \begin{align*}
        p(a \; | \; \vecz^1, \; 1) = \sigma( \theta_1 ) \quad & \text{and} \quad p(a \; | \; \vecz^1, \; 2) = \sigma( \theta_2 ) \\
        p(a \; | \; \vecz^2, \; 1) = \sigma(\theta_3) \quad & \text{and} \quad p(a \; | \; \vecz^2, \; 2) = \sigma( 0 ) = 0.5
    \end{align*}
    We use the confidence score as the maximum token probability, that gives $v_k^{t} = \max_x \pi_{\theta} (x \mid \bm{z}^t, k) = \sigma(|s_{\bm{\theta}}(\bm{z}^t, k)|)$. When $|\mathcal U_t| = 1$, at denoising step $t$, we choose $k^{\text{th}}$ position to decode with the following probability:
    \begin{align*}
        p_{\mathrm{unmask}}(u_1 = k \mid \bm{z}^t) = \frac{\exp\!\big(v_k^{t} / \tau\big)}{\exp\!\big(v_k^{t} / \tau\big) + \sum_{j \in \mathcal{M}_t \setminus \{ k \} } \exp\!\big(v_j^{t} / \tau\big)}.
    \end{align*}
    We now compute the expected reward $J(\bm{\theta})$. As $R(x_1,x_2) = \mathbb{I}\{x_1 = a\}$, we compute the probability of $x_1 = a$ under policy $\pi_{\vectheta}$. The probability of the first position getting selected is $p_{\mathrm{unmask}}(1 \mid \bm{z}^1)$ and then the probability of position $1$ getting filled with $a$ is $\sigma(\theta_1)$. The probability of the second position getting selected at initial stage is $p_{\mathrm{unmask}}(2 \mid \bm{z}^1)$ and irrespective of the value at second position, $x_1 = a$ with probability $\sigma(\theta_3)$. Therefore, the total expected reward is
    \begin{align*}
        J(\vectheta) &= p_{\mathrm{unmask}}(1 \mid \bm{z}^1) \pi_{\theta} (a \mid \vecz^1, 1) + p_{\mathrm{unmask}}(2 \mid \bm{z}^1) \pi_{\theta} (a \mid \vecz^2, 1) \\
        &= \pi_{\theta}(a \mid \vecz^1, 1) + p_{\mathrm{unmask}}(2 \mid \bm{z}^1) ( \pi_{\theta}(a \mid \vecz^2, 1) - p_{\mathrm{val}}(a \mid \vecz^1, 1) ) \\
        &= \sigma(\theta_1) \; + \; p_{\mathrm{unmask}}(2 \mid \bm{z}^1) ( \sigma(\theta_3) - \sigma(\theta_1) )
    \end{align*}

    \paragraph{Zero token-only gradient along direction $v$.} We define the direction $\bm{v} = (0,1,0)$. The token-only gradient along direction $v$ is given by
    \begin{align*}
        \partial_{\bm{\theta}_2} J_{|\mathrm{tok}} =
        \E_{\xi \sim \pi_{\bm{\theta}}} \Big[ \sum_t R(\xi)\,\partial_{\theta_2} \log \pi_{\theta} (x_{u_1} \mid \bm{z}^t, u_1) \Big].
    \end{align*}
    Observe that $\pi_{\theta} (x_{u_1} \mid \bm{z}^t, u_1)$ depends on $\theta_2$ only for $t=1$ and $u_1 = 2$. Therefore, we can write the above as
    \begin{align*}
        \partial_{\bm{\theta}_2} J_{|\mathrm{tok}} &=
        \E_{\xi \sim \pi_{\bm{\theta}}} \Big[  R(\xi) \mathbf{1} \{ u_1 = 2 \} \,\partial_{\theta_2}  \log \pi_{\theta} (x_2 \mid \bm{z}^1, u_1 = 2) \Big] \\
        &= p_{\mathrm{unmask}}(2 \mid \bm{z}^1) \; \E_{\xi \sim \pi_{\bm{\theta}}} \Big[  R(\xi) \,\partial_{\theta_2}  \log \pi_{\theta} (x_2 \mid \bm{z}^1, u_1 = 2) \; \mid \; u_1 = 2 \Big] \\
        &= p_{\mathrm{unmask}}(2 \mid \bm{z}^1) \; \E_{x_2 \sim \pi_{\theta} (\cdot \mid \bm{z}^1, 2)} \; \Big[ \; \partial_{\theta_2}  \log \pi_{\theta} (x_2 \mid \bm{z}^1, u_1 = 2)  \E_{u_2, x_1} \Big[  R(\xi) \; \mid \; u_1 = 2, \; x_2 \Big] \Big] \\
        &= p_{\mathrm{unmask}}(2 \mid \bm{z}^1) \; \sigma( \theta_3 ) \; \E_{x_2 \sim \pi_{\theta}(\cdot \mid \bm{z}^1, 2)} \; \Big[ \; \partial_{\theta_2}  \log \pi_{\theta} (x_2 \mid \bm{z}^1, u_1 = 2) \Big]
    \end{align*}
    We now show that the expectation in the above equation becomes zero. To show that, we write the expectation as sum and then change the order of expectation and derivative. The change in the order is valid because both are with respect to different variables.
    \begin{align*}
        \E_{x_2 \sim \pi_{\theta}(\cdot \mid \bm{z}^1, 2)} \; \Big[ \; \partial_{\theta_2}  \log \pi_{\theta}(x_2 \mid \bm{z}^1, u_1 = 2) \Big]  &= \sum_{x_2}  \pi_{\theta}(x_2 \mid \bm{z}^1, 2) \frac{ \partial_{\theta_2} \pi_{\theta}(x_2 \mid \bm{z}^1, 2) }{ \pi_{\theta}(x_2 \mid \bm{z}^1, 2) } \\
        = \partial_{\theta_2} \sum_{x_2} \pi_{\theta} &  (x_2 \mid \bm{z}^1, 2) = \partial_{\theta_2} 1 = 0
    \end{align*}

    \paragraph{Increasing the true objective $J(\vectheta)$ along $\bm{v}$ direction.} We show that for any parameter $\vectheta$ satisfying $\sigma( \theta_1 ) > \sigma( \theta_3 )$, the gradient along $\theta_2$ improves the performance and $\partial_{\theta_2} J(\vectheta) \neq 0$.

    Observe that for such a parameter vector $\vectheta$, $\sigma(\theta_3)-\sigma(\theta_1) < 0$ and hence $J(\bm{\theta})$ is strictly decreasing as $p_{\mathrm{unmask}}(2 \mid \bm{z}^1)$ increases. Now $p_{\mathrm{unmask}}(2 \mid \bm{z}^1)$ is a strictly increasing function of $v_2^{1} = \sigma(|\theta_2|),$ and $\sigma(|\theta_2|)$ is minimized at $\theta_2=0$ and strictly larger for $\theta_2\neq 0$. Therefore, keeping $\theta_1$ and $\theta_3$ fixed, the objective $J(\bm{\theta})$ is maximized when $\theta_2=0.$ Therefore, for any $\vectheta$ such that $\sigma( \theta_1 ) > \sigma( \theta_3 )$ and $\theta_2 \neq 0$, consider $\vectheta' = \vectheta + \epsilon \bm{v}$ such that $|\theta'_{2}| < |\theta_2|$, the true objective improves:
    \begin{align*}
        J(\vectheta') > J(\vectheta).
    \end{align*}

\paragraph{Nonzero gradient along $\bm{v}$.}
We now show that the true objective exhibits a nonzero directional derivative along $\bm{v}=(0,1,0)$. To this end, we explicitly compute the gradient with respect to $\theta_2$ and verify that it is nonvanishing under the stated conditions.

Differentiating the true objective with respect to $\theta_2$ yields
\[
\partial_{\theta_2} J(\vectheta)
=
\partial_{\theta_2} p_{\mathrm{unmask}}(2 \mid \bm{z}^1)\,
\big(\sigma(\theta_3)-\sigma(\theta_1)\big).
\]

Since $\sigma(\theta_3)-\sigma(\theta_1) < 0$, to conclude that $\partial_{\theta_2} J(\vectheta)\neq 0$, it suffices to establish that
\[
\partial_{\theta_2} p_{\mathrm{unmask}}(2 \mid \vecz^1) \neq 0
\qquad \text{for } \theta_2\neq 0.
\]

We now analyze the dependence of $p_{\mathrm{unmask}}(2 \mid \vecz^1)$ on $\theta_2$. Recall that $v_1^1=\sigma(|\theta_1|), v^1_2=\sigma(|\theta_2|)$, and
\[
p_{\mathrm{unmask}}(2 \mid \vecz^1) = \frac{\exp(v^1_2/\tau)}{\exp(v^1_1/\tau)+\exp(v^1_2/\tau)}.
\]
Differentiating $p_{\mathrm{unmask}}(2 \mid \vecz^1)$ gives
\[
\partial_{\theta_2} p_{\mathrm{unmask}}(2 \mid \vecz^1)
=
\frac{1}{\tau}\,
p_{\mathrm{unmask}}(2 \mid \vecz^1)\big(1-p_{\mathrm{unmask}}(2 \mid \vecz^1)\big)
\,\partial_{\theta_2} v_2^1.
\]
Observe that the multiplicative factor
$ p_{\mathrm{unmask}}(2 \mid \vecz^1)\big(1-p_{\mathrm{unmask}}(2 \mid \vecz^1)\big) $
is strictly positive because $0 < p_{\mathrm{unmask}}(2 \mid \vecz^1) < 1$. Therefore, the only remaining question is whether $\partial_{\theta_2} v_2^1$ vanishes. To compute this term, note that $v_2^1=\sigma(|\theta_2|)$ is a composition of smooth functions away from $\theta_2=0$. Differentiating yields
\[
\partial_{\theta_2} v_2^1
=
\sigma(|\theta_2|)\big(1-\sigma(|\theta_2|)\big)\,\mathrm{sign}(\theta_2),
\qquad \text{for } \theta_2\neq 0.
\]
Since $\sigma(|\theta_2|)\in(0,1)$ for all finite $\theta_2$, the factor $\sigma(|\theta_2|)(1-\sigma(|\theta_2|))$ is strictly positive, and $\mathrm{sign}(\theta_2)\neq 0$ whenever $\theta_2\neq 0$. Hence,
$\partial_{\theta_2} v_2^1 \neq 0$ for all $\theta_2\neq 0.$

Combining the above observations, we conclude that
\[
\partial_{\theta_2} p_{\mathrm{unmask}}(2 \mid \vecz^1) \neq 0
\qquad \text{for all } \theta_2\neq 0,
\]
and therefore
\[
\partial_{\theta_2} J(\vectheta)\neq 0
\qquad \text{whenever } \sigma(\theta_1)>\sigma(\theta_3) \text{ and } \theta_2\neq 0.
\]

Finally, since $\bm{v}=(0,1,0)$, the directional derivative along $\bm{v}$ is given by
\[
\langle \nabla_{\vectheta} J(\vectheta), \bm{v}\rangle
=
\partial_{\theta_2} J(\vectheta).
\]
Thus,
\[
\langle \nabla_{\vectheta} J(\vectheta), \bm{v}\rangle \neq 0,
\]
which establishes that the true objective is not stationary along $\bm{v}$.

\end{proof}

\end{document}